\documentclass{article}

\usepackage{arxiv_style,times}
\usepackage{amsmath,amssymb,amsthm,mathtools}
\usepackage{algorithm}
\usepackage[noend]{algpseudocode}
\usepackage{array,booktabs}
\usepackage{graphicx}
\usepackage{enumitem}
\usepackage{url}
\usepackage[colorlinks=true,citecolor=blue,linkcolor=blue,urlcolor=blue]{hyperref}
\usepackage[nameinlink,capitalise]{cleveref}

\title{Nonparametric Contextual Pricing and Inventory Learning under Censored Demand}
\author{%
Zean Han\textsuperscript{*} \quad
Jing Liang\textsuperscript{*} \quad
Ruihan Lin \quad
Zezhen Ding\textsuperscript{\textdagger} \quad
Jiheng Zhang\\[0.4em]
{\normalfont Department of Industrial Engineering and Decision Analytics}\\
{\normalfont The Hong Kong University of Science and Technology}\\[0.4em]
{\normalfont\small\texttt{zhanax@connect.ust.hk}} \quad
{\normalfont\small\texttt{jliangcd@connect.ust.hk}} \quad
{\normalfont\small\texttt{rlinah@connect.ust.hk}}\\
{\normalfont\small\texttt{zdingah@connect.ust.hk}} \quad
{\normalfont\small\texttt{jiheng@ust.hk}}
}
\date{}

\newtheorem{theorem}{Theorem}
\newtheorem{corollary}{Corollary}
\newtheorem{proposition}{Proposition}
\newtheorem{lemma}{Lemma}
\newtheorem{assumption}{Assumption}
\newtheorem{definition}{Definition}
\newtheorem{remark}{Remark}
\crefname{assumption}{Assumption}{Assumptions}
\Crefname{assumption}{Assumption}{Assumptions}
\makeatletter
\providecommand{\theHALG@line}{}
\renewcommand{\theHALG@line}{\thealgorithm.\arabic{ALG@line}}
\makeatother

\newcommand{\R}{\mathbb{R}}
\newcommand{\E}{\mathbb{E}}
\newcommand{\Pbb}{\mathbb{P}}
\newcommand{\cX}{\mathcal{X}}
\newcommand{\cA}{\mathcal{A}}
\newcommand{\cI}{\mathcal{I}}
\newcommand{\cP}{\mathcal{P}}
\newcommand{\Reg}{\mathsf{Reg}}
\newcommand{\argmax}{\operatorname*{arg\,max}}
\newcommand{\argmin}{\operatorname*{arg\,min}}
\newcommand{\polylog}{\operatorname{polylog}}
\newcommand{\one}{\mathbf{1}}
\newcommand{\norm}[1]{\left\lVert #1 \right\rVert}
\newcommand{\ProxyQ}{\overline Q}
\newcommand{\ProxyG}{\overline G}

\begin{document}
\maketitle
\begin{NoHyper}
\begingroup
\renewcommand{\thefootnote}{}
\footnotetext{\textsuperscript{*} Equal contribution. \quad
\textsuperscript{\textdagger} Corresponding author.}
\endgroup
\end{NoHyper}
\vspace{-0.35in}

\begin{abstract}
In online retailing, when a product sells out, a retailer often sees only the units sold, not how many customers
would have bought it had inventory been available.  However, the inventory level determines how
much demand is revealed, and this information can influence subsequent decisions and future profits.
We study an online selling problem in which, in each round, the
seller observes a market context and then makes pricing and stocking decisions based on censored
sales data from previous rounds.  The challenge is to learn a context-dependent pricing and
stocking policy without assuming a particular formula for demand or observing realized profit.
To overcome this difficulty, we propose a Surrogate-Profit UCB (SP-UCB) algorithm that
turns each incomplete sales record into a reliable guide for both inventory and price decisions,
using data from past rounds with similar market conditions.  This design allows us to
learn while serving customers, without a separate exploration phase or the need to recover all
demand hidden by stockouts.  We prove the minimax optimality of the proposed algorithm, with
strictly faster rates when expected profit varies more smoothly with price.  Comprehensive numerical
experiments have been conducted to confirm the effectiveness of the proposed algorithm.
\end{abstract}

\section{Introduction}

Online retailers often need to choose both a price and an inventory level
before demand is known.  Flash-sale and fashion retailers, for example, sell
short-lived assortments and face changing seasons and local markets.  They
use limited sales records to forecast demand and set prices
\citep{ferreira2016analytics,caro2012clearance}.  Price and inventory are
closely linked: a lower price can increase demand and the risk of a stockout,
while a cautious inventory level can reduce leftover costs but also hide how
much demand the price would have generated.

The data generated by these operations are censored.  When a product sells
out, observed sales equal the available inventory, leaving unmet demand
unknown \citep{musalem2010structural,conlon2013demand}.  Because lost-sales
costs depend on this unobserved demand, the retailer cannot fully observe
realized profit.  Learning from sales therefore requires a way to compare
the profitability of pricing and stocking decisions despite the missing demand
\citep{besbes2013censoring,chen2021nonparametric,chen2024optimal}.

The problem becomes harder when decisions must adapt to context.  The same
price may work well in one season or local market and poorly in another.
Pooling all past sales would ignore these differences, while learning
separately in every context would leave too little data.  Even at the same
price and in similar contexts, different inventory levels can produce
different observed sales.  The seller must therefore share information
across similar contexts while accounting for the censoring caused by past
inventory decisions.

Existing online studies of censored pricing and inventory learning focus on
a fixed demand environment \citep{chen2021nonparametric,chen2024optimal},
while contextual approaches often use historical data collected before
deployment \citep{tang2025offline,gundem2025offline}.  Recent work also
studies online learning under contextual and censored demand, using a finite
collection of basis functions and an initial high-inventory exploration
phase \citep{han2026pricingstock}.  Our focus is on learning under a
nonparametric demand model, without requiring inventories high enough to
reveal all demand.  This raises a natural question:

\emph{How can an online seller use censored sales to learn which prices and
inventory levels are best for each context when realized profit is not fully
observed, and what is the best achievable learning rate?}

We answer this question through three contributions.

\begin{enumerate}[leftmargin=2em]

  \item \textbf{An observable learning signal under sales censoring.}
  With sales-only feedback, the seller cannot tell whether observed sales
  reflect true demand or the inventory cap.  We construct an observable score
  from sales and leftover inventory.  At a fixed price, this score has the
  same best inventory choice as true profit.  For comparing prices, however,
  the score has a price-dependent bias.  We remove that bias using a
  low-demand quantile that the algorithm makes visible.  The resulting signal
  contains the information needed for both stocking and pricing decisions.

  \item \textbf{Context-aware learning with censored sales.}
  We develop the Surrogate-Profit UCB (SP-UCB) algorithm for joint
  pricing and stocking.  Each inventory decision makes the calibration
  quantile observable, while context-weighted sales records provide the
  information used to evaluate prices.  All learning takes place during
  deployment, using the sales information generated by regular operations.

  \item \textbf{Minimax-optimal regret rates.}
  We prove matching upper and lower bounds for the proposed method.  Under
  Lipschitz price values, the minimax regret is
  \(\widetilde\Theta\!\left(T^{\frac{d+2\alpha}{d+3\alpha}}\right)\), while
  under twice-smooth price values it becomes
  \(\widetilde\Theta\!\left(T^{\frac{2d+3\alpha}{2d+5\alpha}}\right)\).
  These rates quantify two sources of difficulty: a higher context dimension
  makes local learning harder, while a smoother expected-profit curve over
  price allows a coarser price grid and faster learning.

\end{enumerate}

\section{Related Work}
\label{sec:related}

\paragraph{\textbf{Censored inventory and joint pricing--inventory learning.}}
Learning inventory decisions from censored sales has a long history.
Existing methods use sales and stockout information to learn stocking
policies without observing unmet demand
\citep{huh2009nonparametric,huh2011adaptive,besbes2013censoring}, with
subsequent work incorporating contextual features \citep{ding2024feature}.
Joint pricing introduces the additional need to compare profits across
prices.  In a noncontextual setting, \citet{chen2021nonparametric} construct
a spline approximation of the demand--price function, while
\citet{chen2024optimal} estimate profit differences between prices through
an inversion of empirical distributions.  Our approach first constructs an
observable sales score whose expectation preserves inventory profit gaps at
a fixed price.  Under a residual law shared across prices within each
context, a visible low quantile shifts with mean demand and can therefore
remove the score's price-dependent bias.  The resulting calibrated target
preserves profit gaps across prices.  We estimate this target locally across
contexts, allowing the residual distribution to vary with context without
recovering the unobserved demand tail.

\paragraph{\textbf{Contextual pricing and bandit learning.}}
Contextual pricing methods use customer or product features to learn how
demand responds to prices, often through structured demand models
\citep{cohen2020feature,javanmard2019dynamic,ban2021personalized}.
Nonparametric pricing exploits smoothness across prices or contexts
\citep{chenGallego2021nonparametric,wang2021multimodal}.  Contextual bandit
methods exploit similarity \citep{slivkins2014contextual}, while
continuous-action methods also study smoothed policy benchmarks
\citep{krishnamurthy2020continuous}.  For offline-to-online linear contextual
bandits, \citet{han2026direction} use directional bias bounds to guide the
reuse of historical data.  Related work studies dynamic pricing
under inventory constraints \citep{chen2022primal} and revenue-based
financing \citep{lin2026financing}.  In our setting, the inventory chosen in each period
also determines how much demand becomes visible.  When a stockout occurs,
lost-sales costs are unobserved, so realized profit cannot be used directly
as a learning signal.  Optimistic price selection therefore requires both a
calibrated signal and inventory decisions that make it estimable.  Inventory
estimates need only be accurate enough to support price comparison: near the
optimal stocking level, profit loss is quadratic in inventory error.
SP-UCB exploits this property to learn inventories on a coarser context
partition, reducing the cost of inventory learning while controlling the
error in price evaluation.

\paragraph{\textbf{Contextual joint pricing and inventory learning.}}
Historical-data methods address feature-based inventory control
\citep{ban2019bigdata} and joint pricing--inventory optimization
\citep{qin2022data}.  With censored sales, \citet{bu2023offline} learn prices
for a given inventory, \citet{tang2025offline} learn feature-based pricing
rules, and \citet{gundem2025offline} study sequential joint pricing--inventory
policies.
The closest online work is \citet{han2026pricingstock}, which represents
demand using known basis functions with unknown coefficients and also
analyzes approximation error.  Its algorithm uses an initial high-inventory
exploration phase to obtain demand observations before committing to an
estimated policy.  We study learning guarantees over nonparametric
smoothness classes in context and price.  Our policy requires visibility of
a low calibration quantile rather than full demand, and continues to update
price estimates after local inventory profiling.  Our analysis focuses on
within-period stocking decisions, whereas \citet{han2026pricingstock} also
account for inventory carryover.  For this setting, we establish matching
upper and lower regret bounds that characterize the dependence on context
dimension, contextual smoothness, and price-value regularity.

\section{Problem Setup}
\label{sec:setup}

\paragraph{Decisions and feedback.}
Let \(\cX=[0,1]^d\) be the context space and
\(\cP=[p_{\min},p_{\max}]\), with
\(0<p_{\min}<p_{\max}<\infty\), the price range.
The seller chooses actions from
\(\cA=\cP\times[y_{\min},y_{\max}]\), where
\(0\le y_{\min}<y_{\max}<\infty\). The action ranges and per-unit
lost-sales and holding costs \(b,h_c>0\) are known.

In each period \(t=1,\ldots,T\), the seller observes \(x_t\) and chooses
a price \(p_t\) and inventory \(y_t\). A policy \(\pi\) maps the current
context and observed history
\(\mathcal H_{t-1}=\{(x_s,p_s,y_s,O_s):s<t\}\) to a possibly randomized
action \((p_t,y_t)\in\cA\). Inventory is replenished for each period,
with no carryover. Demand then realizes as
\[
  D_t=\lambda(x_t,p_t)+\epsilon_t,
\]
where the demand location \(\lambda\) and residual distribution are unknown.
The seller observes sales \(O_t=\min\{D_t,y_t\}\).
Sales reveal demand when \(O_t<y_t\). A stockout, identified by
\(O_t=y_t\), reveals only that \(D_t\ge y_t\).

\paragraph{Profit and learning objective.}
The one-period profit is
\[
  r((p,y),D)=p\min\{D,y\}-b(D-y)^+-h_c(y-D)^+.
\]
Lost-sales costs depend on unmet demand, so realized profit is partially
unobserved when a stockout occurs. For context \(x\), write the expected
profit as
\[
  Q_x(p,y)=\E[r((p,y),D)\mid x,p].
\]
We compare the policy with a seller who knows the conditional demand
distribution and maximizes expected profit for each observed context.
This benchmark chooses its action before demand realizes. Since inventory
is replenished between periods, its optimization separates across periods.
Optimizing inventory at a given price and then choosing the price gives
\[
  G_x(p)=\max_{y\in[y_{\min},y_{\max}]}Q_x(p,y),
  \qquad
  V_x=\max_{p\in\cP}G_x(p).
\]
Regret is the expected cumulative profit loss relative to this benchmark:
\[
  \Reg_T(\pi)
  =\E_\pi\!\left[\sum_{t=1}^T
    \{V_{x_t}-Q_{x_t}(p_t,y_t)\}\right].
\]
The expectation is over contexts, demand shocks, and policy randomization.
The objective is to minimize this loss using the censored sales generated
by the policy's own decisions.

\paragraph{Demand structure and regularity.}
We first specify how contexts and demand shocks arrive. We then impose
smoothness for sharing information across contexts and prices, followed by
conditions on the quantiles used for inventory learning and calibration.

\begin{assumption}[Demand process]
\label{ass:context}
\textbf{(i) Arrivals.} The pairs \((x_t,\epsilon_t)\) are i.i.d.\ across
periods and independent of policy randomization. Contexts follow an
arbitrary distribution on \(\cX\).

\textbf{(ii) Conditional shocks.} Given \(x_t=x\), the shock \(\epsilon_t\)
has distribution \(F_x\), with mean zero, a density \(f_x\), and common compact support
\([\underline z,\overline z]\). The location surface \(\lambda\) is
uniformly bounded, and demand is nonnegative almost surely.
\end{assumption}

Thus \(\lambda(x,p)=\E[D\mid x,p]\). At a fixed context, prices shift
the demand location while the residual distribution stays fixed. Sharing
observations across contexts also requires control of changes in this
distribution.

\begin{assumption}[Context and price smoothness]
\label{ass:holder}
\textbf{(i) Context variation.} For some \(0<\alpha\le1\) and finite
\(L_x,L_F\), uniformly in \(x,x'\in\cX\) and \(p\in\cP\),
\[
\begin{aligned}
  |\lambda(x,p)-\lambda(x',p)|
  &\le L_x\norm{x-x'}^\alpha,\\
  \sup_{u\in\R}|F_x(u)-F_{x'}(u)|
  &\le L_F\norm{x-x'}^\alpha.
\end{aligned}
\]
\textbf{(ii) Price variation.} For a finite \(L_p\),
\(|\lambda(x,p)-\lambda(x,q)|\le L_p|p-q|\) uniformly in
\(x\in\cX\) and \(p,q\in\cP\).
\end{assumption}

Inventory learning uses the demand quantile associated with the critical
ratio, while price calibration uses a lower quantile. To describe their
regularity, define \(F_x^{-1}(u)=\inf\{z:F_x(z)\ge u\}\) for
\(u\in(0,1)\), and write
\[
  z_u(x)=F_x^{-1}(u),\qquad
  q_u(x,p)=\lambda(x,p)+z_u(x),\qquad
  \phi(p)=\frac{p+b}{p+b+h_c},
\]
and abbreviate \(z_p(x)=z_{\phi(p)}(x)\). Fix a calibration level
\(\rho\in(0,\inf_{p\in\cP}\phi(p))\).

\begin{assumption}[Quantile regularity and inventory feasibility]
\label{ass:inventory}
\textbf{(i) Density bounds.} There are \(\kappa,\kappa',r_q>0\) such that
\(0\le f_x(u)\le\kappa'\) on the common support, uniformly in \(x\).
Moreover, for every \(x\) and \(p\),
\[
  f_x(u)\ge\kappa
  \quad\text{whenever}\quad
  |u-z_\rho(x)|\le r_q
  \quad\text{or}\quad |u-z_p(x)|\le r_q.
\]
\textbf{(ii) Inventory margins.} There is \(r_y\in(0,r_q]\) such that,
uniformly in \(x,p\),
\[
  y_{\min}+r_y\le q_\rho(x,p),
  \qquad
  q_{\phi(p)}(x,p)\le y_{\max}-r_y.
\]
\end{assumption}

These conditions yield the unique optimal inventory
\(y_x^*(p)=q_{\phi(p)}(x,p)\) and a uniform gap
\(y_x^*(p)-q_\rho(x,p)\ge m_q>0\)
(\cref{lem:unique-inventory,lem:quantile-separation}). This gap allows
inventory estimates close to \(y_x^*(p)\) to expose the calibration
quantile in sales. The price smoothness in \cref{ass:holder} also implies
that \(G_x\) is uniformly Lipschitz (\cref{lem:price-lipschitz}).
The faster-rate specialization in \cref{sec:results} imposes additional
smoothness on \(G_x\).

\section{Learning and Deciding with a Surrogate Reward}
\label{sec:algorithm}

Sales censoring makes unfulfilled demand unobservable, yet the occurrence of a stockout remains observable. Furthermore, censoring only truncates the upper tail of the demand distribution without distorting the shape of its observable, uncensored portion. Motivated by these observations, we propose a two-tier surrogate evaluation framework that preserves expected profit differences for inventory and price comparisons within a fixed context.

Since exact context repetitions are rare in a continuous context space, a data-driven approach must draw on observations from neighboring contexts. Doing so introduces a trade-off between local accuracy and data sufficiency, alongside the exploration-exploitation trade-off inherent in sequential learning. A kernel-based UCB framework naturally addresses both challenges. Moreover, the structural relationship between price and inventory allows inventory information to be shared across nearby prices and contexts, reducing exploration costs while controlling the resulting expected profit loss.

Building on these insights, we introduce the Surrogate-Profit UCB (SP-UCB) algorithm.

The policy is supplied with \(T\), the context smoothness exponent
\(\alpha\), the calibration level \(\rho\), and the choice of the
Lipschitz or smooth-price regime. We select a kernel and tuning constants
using known numerical bounds for the model class. The required bounds and
parameter choices are specified in \cref{app:tuning}.

\subsection{Constructing a Two-Tier Surrogate Profit from Observable Data}

For a fixed context, we solve the censoring problem in two steps. First, an observable sales-based surrogate differs from the expected profit by an offset. Since this offset ignores inventory, the surrogate perfectly captures the profit gaps between different inventory choices at a fixed price. Second, although this offset changes across prices, an observable low-demand quantile shifts with price by the exact same amount as mean demand. We obtain a price-independent offset by subtracting this quantile multiplied by the per-unit lost-sales cost (b). Ultimately, these two steps create a two-tier surrogate that evaluates inventory levels first and prices second.

We decouple the true profit into an observable component and an unobservable remainder. For price \(p\),
inventory \(y\), and observed sales \(O=\min\{D,y\}\), where \(D\) is demand, define surrogate score \(Y(p,y,O)\) from observable information  
\[
  Y(p,y,O):=(p+b)O-h_c(y-O),
\]
where \(b\) and \(h_c\) are the per-unit lost-sales and holding costs.
The realized profit can be reformulated as
\(r((p,y),D)=Y(p,y,O)-bD\).  Let \(\ProxyQ_x(p,y)\) denote the expected
surrogate score at context \(x\).  Its expectation is
\[
  \E[Y(p,y,O)\mid x,p,y]
  =\ProxyQ_x(p,y)=Q_x(p,y)+b\lambda(x,p),
\]
where \(Q_x(p,y)\) is the expected profit and
\(\lambda(x,p)=\E[D\mid x,p]\) by the demand model and the
zero-mean noise condition in \cref{ass:context}.
Thus, finding the price--inventory pair with the highest expected profit can
be approached by computing \(Y\) from observable data, obtaining information
about the mean-demand offset \(b\lambda(x,p)\), and maximizing the corrected value.

For a given context \(x\), each price \(p\) admits an optimal
inventory level \(y_x^*(p)\) that maximizes the expected surrogate score. The
mean-demand offset \(b\lambda(x,p)\) is independent of inventory decisions.
Therefore, although the full profit level remains unknown, expected profit differences between inventory decisions at a fixed context and price, including the loss from deviating from the optimal inventory, can already be computed through the surrogate, as detailed in \cref{prop:proxy-identity}. 

Let
\(S_x(y\mid p):=\Pbb(D\ge y\mid x,p)\) denote the stockout probability and
\(\tau(p):=h_c/(p+b+h_c)\) the critical stockout threshold \citep{arrow1951optimal}.

Then
\[
  \partial_y \E[Y(p,y,O)\mid x,p]
  =(p+b+h_c)\bigl(S_x(y\mid p)-\tau(p)\bigr).
\]
Since \(S_x(y\mid p)\) is nonincreasing in \(y\), so is this derivative.
When the stockout probability exceeds the critical threshold, increasing
inventory improves expected profit; when it is below the threshold,
decreasing inventory improves expected profit.
Moreover, controlling the gap \(\lvert S_x(y\mid p)-\tau(p)\rvert\)
controls the inventory error locally, as formalized below.

\begin{lemma}[Inventory optimality]
\label{lem:inventory-consequences}
The optimal inventory is
\(y_x^*(p)=\lambda(x,p)+z_{\phi(p)}(x)\). It uniquely solves
\(S_x(y\mid p)=\tau(p)\) and maximizes
\(\E[Y(p,y,O)\mid x,p]\).
There exist constants \(c_y,C_y>0\) such that, whenever
\(\lvert y-y_x^*(p)\rvert\le r_y\),
\[
  c_y\lvert y-y_x^*(p)\rvert
  \le \lvert S_x(y\mid p)-\tau(p)\rvert
  \le C_y\lvert y-y_x^*(p)\rvert,
\]
and
\[
  G_x(p)-Q_x(p,y)
  \le C_y\lvert y-y_x^*(p)\rvert^2.
\]
\end{lemma}

The analysis above characterizes the optimal inventory for a fixed context and price. However, learning in a continuous context space cannot generally rely on repeated observations at exactly the same context, motivating the use of data from nearby contexts.

We partition the context space into a context grid and pool observations
within each cell to learn its inventory decision.

\begin{definition}[Context grid]
\label{def:context-grid}
A context grid at resolution \(\delta_x>0\) is a finite measurable
partition \(\mathcal C \) of \(\mathcal X\) satisfying
\[
  \operatorname{diam}(C)\le \delta_x
  \quad\text{for every }C\in\mathcal C,
  \qquad
  |\mathcal C|
  \le C_{\rm cell}\delta_x^{-d},
\]
where \(C_{\rm cell}\) is independent of \(\delta_x\).
\end{definition}

A sufficiently fine context grid allows a single inventory decision
to approximate the optimal inventory throughout each cell (\cref{lem:context-grid}).
We therefore propose a bisection procedure to learn the inventory target
within each context cell. At each selected price, the procedure compares
the pooled empirical stockout probability with the critical threshold
to determine the direction of inventory adjustment.
The full procedure is given in \cref{alg:profiler}.

\begin{proposition}[Inventory learning]
\label{prop:inventory-learning}
Fix a price \(p\), a target value accuracy \(\varepsilon>0\),
and a context grid with \(\delta_x^\alpha\le\sqrt{\varepsilon}\).
For sufficiently small \(\varepsilon\), with probability at least
\(1-\delta\), the estimates returned by \cref{alg:profiler}
satisfy, simultaneously for every completed cell \(C\),
\[
  \sup_{x\in C}
  \left|\widehat y_C(p)-y_x^*(p)\right|
  \le C_{\rm inv}\sqrt{\varepsilon},
\]
where \(C_{\rm inv}>0\) depends only on the model constants and fixed
algorithm parameters, not on \(T\), \(\varepsilon\), or \(\delta_x\).

The procedure uses at most \(\widetilde O(\varepsilon^{-1})\)
rounds per cell. The cumulative regret incurred during inventory search
in each completed cell is at most
\(
  \widetilde O\!\left(\varepsilon^{-1}\right).
\)
\end{proposition}

Inventory profiling also enables price comparison by making the low quantile
\(q_\rho(x,p)\) observable below the certified inventory
\(\widehat y_C(p)\) (\cref{lem:profile}). Although this inventory already
supports stocking decisions at a fixed price, comparing prices still requires
correcting the price-dependent offset \(b\lambda(x,p)\) in the surrogate
score. One intuitive approach is to estimate \(\lambda(x,p)\) empirically
by setting inventory relatively high to reduce lost sales. However, this
approach incurs additional inventory costs, and the feasible inventory range
may not be sufficient to reveal all demand. Instead, the
visible low quantile shifts with price by exactly the same amount as mean
demand at a fixed context. It therefore provides the required correction
without recovering the unobserved demand tail.

Recall that the expected surrogate score takes the form $\E[Y(p,y,O)\mid x,p]=Q_x(p,y)+b\lambda(x,p)$. The demand model implies $q_\rho(x,p)=\lambda(x,p)+z_\rho(x)$, where $z_\rho(x)$ is the $\rho$-quantile of the price-independent residual noise. Subtracting $bq_\rho(x,p)$ from the expected surrogate isolates the price-dependent terms:
$$   \E[Y(p,y,O)\mid x,p]-bq_\rho(x,p)=Q_x(p,y)-bz_\rho(x).  $$
Crucially, the remaining offset $bz_\rho(x)$ is invariant to price and thus preserves the true profit landscape for price comparison.
Building on this price-invariance, we evaluate prices using the optimized surrogate value $\ProxyG_x(p):=\max_y\E[Y(p,y,O)\mid x,p]$. We define the corrected surrogate profit function as:
$$   H_x(p):=\ProxyG_x(p)-bq_\rho(x,p). $$
It follows that $H_x(p)=G_x(p)-bz_\rho(x)$. Because $H_x(p)$ and the true expected profit $G_x(p)$ differ only by a context-dependent constant, they share the exact same set of optimal prices. Maximizing $H_x(p)$ thus recovers the optimal price, which can then be paired with its certified inventory level to yield a complete joint decision.

\subsection{Maximizing the Surrogate Profit with Contextual Information}
Although contextual smoothness provides limited guarantees for transferring information across nearby contexts, effective use of contextual data requires a trade-off between the amount of data available in the neighborhood and the accuracy of the resulting information, which is further compounded by the exploration-exploitation trade-off inherent in sequential learning.

To navigate these challenges, we propose a kernel-based Upper Confidence Bound (UCB) framework. To compare candidate prices using UCB, we restrict our search for the optimal price to a finite candidate price set $\Gamma$, specified in \cref{def:candidate-price-set} in the appendix. An intuitive approach is to run Inventory Search (\cref{alg:profiler}) separately for every candidate price. However, inventory learning need only be accurate enough to support price comparison. Near the optimal inventory, small inventory errors have only a second-order effect on expected profit. Since the optimal inventory also varies smoothly with price, an inventory learned at one price can serve nearby prices while maintaining the desired profit accuracy.

Exploiting this structure, we perform inventory search on a coarser price grid $\Gamma_y$, and let each candidate price $p\in\Gamma$ use the inventory learned at its nearest grid point in $\Gamma_y$.

\begin{proposition}
For Lipschitz price values and sufficiently small \(\varepsilon\),
applying \cref{alg:profiler} on \(\Gamma=\Gamma(\varepsilon)\) or
\(\Gamma_y=\Gamma(\sqrt{\varepsilon})\) yields 
\[
  \sup_{x\in C,\;p\in\Gamma}
  \bigl\{G_x(p)-Q_x(p,\widehat y_C(p))\bigr\}
  \le C_{\rm share}\varepsilon
\]
for every completed cell \(C\) with probability
at least \(1-\delta\). \(C_{\rm share}>0\) is
independent of \(T\) and \(\varepsilon\).
The respective per-cell inventory-search regret bounds are
\(\widetilde O(\varepsilon^{-2})\) for \(\Gamma\),
\(\widetilde O(\varepsilon^{-3/2})\) for \(\Gamma_y\).
\end{proposition}

With the shared inventory decisions in place, we estimate the surrogate profit at the current context by leveraging nearby observations. A kernel restricts this pooling to a neighborhood whose bandwidth controls the localization bias. Let \(h>0\) denote the bandwidth and let \(K:\mathbb R^d\to[0,1]\) be a compactly supported Lipschitz kernel, with \(K(u)=0\) for \(\|u\|>1\) and \(K(u)\ge k_0>0\) for \(\|u\|\le1/2\).

For each \(p\in\Gamma\), let \(\mathcal T_t(p)\) contain the rounds
before \(t\) in which price \(p\) was selected after
\cref{alg:profiler} completed the inventory search for all prices
in \(\Gamma_y\) within the corresponding context cell.
Let \(x_s\) denote the context observed in round \(s\), and let \(x_t\)
denote the current context. For each \(s\in\mathcal T_t(p)\), the
normalized weight is
\[
  \frac{K((x_s-x_t)/h)}{W_t(x_t,p)},
  \qquad
  W_t(x_t,p):=\sum_{s\in\mathcal T_t(p)}K((x_s-x_t)/h),
\]
which is defined for \(W_t(x_t,p)>0\).

Under this weighting scheme, Let \(\widehat{\ProxyG}_t(x_t,p)\) be the weighted average of
observed surrogate scores, and \(\widehat q_{\rho,t}(x_t,p)\) be the
\(\rho\)-quantile of the empirical sales distribution formed with the same
weights.  These two estimates approximate the surrogate value and demand
quantile at the current context.  With sufficient amount of neighboring observations,
their difference,
\[
  \widehat H_t(x_t,p)
  :=\widehat{\ProxyG}_t(x_t,p)-b\widehat q_{\rho,t}(x_t,p),
\]
approximates the learning target \(H_x(p)\).

The estimate \(\widehat H_t(x_t,p)\) is subject to three sources of error: the profit loss from using learned and shared inventories, the localization bias from pooling nearby contexts, and the sampling uncertainty in the weighted observations. We account for these errors through a confidence width and select prices optimistically. This allows a sparsely sampled price to remain competitive even when its current estimated value is low. For contextual H\"older exponent \(\alpha\), choose a bandwidth \(h>0\)
satisfying \(h^\alpha=\varepsilon\). For a sufficiently large fixed
constant \(C>0\), define the confidence width as
\[
  B_t(x_t,p):=C\left[
    \varepsilon+h^\alpha+
    \sqrt{\frac{\log(T|\Gamma|/h)}{\max\{1,W_t(x_t,p)\}}}
  \right].
\]

The upper confidence index is
\[
  U_t(x_t,p):=\widehat H_t(x_t,p)+B_t(x_t,p),
\]
with \(U_t(x_t,p)=+\infty\) for \(W_t(x_t,p)=0\).

Once the inventory search is complete for \(\Gamma_y\) within the current context cell, the algorithm selects \(p_t\in\argmax_{p\in\Gamma}U_t(x_t,p)\) and uses the inventory associated with the nearest grid point to \(p_t\) in \(\Gamma_y\). The resulting sales observation updates the surrogate profit estimate and confidence width for that price at nearby contexts, thereby refining the upper confidence indices used in subsequent rounds. The complete SP-UCB algorithm is presented in \cref{alg:mck}.

Detailed error analysis, algorithmic procedures, and parameter choices
are provided in the appendix
\cref{app:kernel-concentration,app:upper-complete,app:algorithms,app:tuning}.

\section{Main Results}
\label{sec:results}

Let \(\mathfrak I_{\rm Lip}\) denote the instance class satisfying
\cref{ass:context,ass:holder,ass:inventory}, with fixed numerical bounds
on the model constants. Its price values are uniformly Lipschitz by
\cref{lem:price-lipschitz}. A smoother price-value curve yields a faster
rate under the following additional condition.

\begin{assumption}[Smooth-price specialization]
\label{ass:smooth-price}
For every \(x\), \(G_x\) is twice differentiable in price, with
\(\sup_{x,p}|\partial_{pp}G_x(p)|\le L_{G,2}\) for a fixed finite
\(L_{G,2}\).
\end{assumption}

Let \(\mathfrak I_{\rm sm}\) be the subclass satisfying
\cref{ass:smooth-price}. The formal upper and lower statements, including
the fixed numerical envelopes and tuning constants, are given in
\cref{app:formal-results}.  We use the bandwidth \(h\), target value
accuracy \(\varepsilon=h^\alpha\), and candidate price set \(\Gamma\)
from \cref{sec:algorithm}.  The fixed multiplier \(c_\eta>0\) sets the
stockout-probability tolerance \(\eta=c_\eta\sqrt{\varepsilon}\) in
the inventory search, as specified in \cref{app:tuning}.

\begin{theorem}[Upper bounds for contextual censored sales]
\label{thm:main-upper}
For fixed \(b,h_c>0\) and any fixed \(c_\eta>0\), with certified tuning
as in \cref{app:tuning}, \(\mathsf{SP}\)-UCB satisfies, up to logarithmic
factors,
\[
  \sup_{\cI\in\mathfrak I_{\rm Lip}}\Reg_T(\mathsf{SP}\text{-}\mathsf{UCB};\cI)
  \le
  \widetilde O\!\left(T^{\frac{d+2\alpha}{d+3\alpha}}\right),
  \qquad
  \sup_{\cI\in\mathfrak I_{\rm sm}}\Reg_T(\mathsf{SP}\text{-}\mathsf{UCB};\cI)
  \le
  \widetilde O\!\left(T^{\frac{2d+3\alpha}{2d+5\alpha}}\right).
\]
\end{theorem}

\begin{remark}
The multiplier \(c_\eta\) is fixed independently of \(T\), and the hidden
constants may depend on it.
The upper bound follows the learning steps in \cref{sec:algorithm}.
The algorithm first searches for inventories within context cells.  Once
this search is complete in a cell, the returned inventories make the low
quantile observable, allowing the algorithm to estimate \(H_x(p)\) and
compare candidate prices using \(U_t(x_t,p)\).  Since \(H_x(p)\) preserves
price gaps, errors in this comparison translate into expected profit loss.
Suppressing logarithms and constants, the regret bound has the form
\[
  h^{-d/2}\varepsilon^{-3/2}
  \;+\; T\varepsilon
  \;+\; \sqrt{T|\Gamma|h^{-d}} .
\]
The first term bounds the cost of inventory search over the context cells
and inventory-search prices specified in \cref{app:algorithms}.
The second collects inventory approximation, context localization, and
price discretization errors, each of order \(\varepsilon\) per subsequent
round.  The third bounds accumulated statistical uncertainty in the
kernel estimates.  By \cref{def:candidate-price-set},
\(|\Gamma|\asymp\varepsilon^{-1}\) for Lipschitz price values and
\(|\Gamma|\asymp\varepsilon^{-1/2}\) for twice-smooth price values.
\end{remark}

\begin{theorem}[Lower bounds for contextual censored sales]
\label{thm:main-lower}
For fixed \(b,h_c>0\), every policy incurs, up to logarithmic factors,
\[
  \inf_\pi\sup_{\cI\in\mathfrak I_{\rm Lip}}\Reg_T(\pi;\cI)
  \ge
  \widetilde\Omega\!\left(T^{\frac{d+2\alpha}{d+3\alpha}}\right),
  \qquad
  \inf_\pi\sup_{\cI\in\mathfrak I_{\rm sm}}\Reg_T(\pi;\cI)
  \ge
  \widetilde\Omega\!\left(T^{\frac{2d+3\alpha}{2d+5\alpha}}\right).
\]
\end{theorem}

\begin{remark}
The lower bounds match the upper-bound exponents over the stated instance
classes.  They show that no policy using the same censored-sales feedback
can achieve a smaller worst-case polynomial order of regret.
The constructions and information bounds are given in
\cref{app:formal-results}.  This comparison establishes the optimality of
the overall rates; it does not require every policy to use the same
inventory search or incur each term of the upper-bound decomposition.
\end{remark}

\paragraph{Interpreting the exponents.}
A smaller bandwidth \(h\) reduces localization bias but leaves fewer
nearby observations for each price.  A smaller target value error
\(\varepsilon\) also requires more accurate inventory decisions and a
finer candidate price set.  With \(\varepsilon=h^\alpha\), balancing
\(T\varepsilon\) and \(\sqrt{T|\Gamma|h^{-d}}\) gives
\[
  h\asymp T^{-1/(d+3\alpha)}
  \quad\text{for Lipschitz price values},
  \qquad
  h\asymp T^{-2/(2d+5\alpha)}
  \quad\text{for twice-smooth price values},
\]
up to logarithmic factors.  At these choices, the inventory-search term
is of lower order than the resulting regret bound.
Larger \(d\) worsens the exponent because observations are spread across
more context neighborhoods, while larger \(\alpha\) allows more sharing
with less bias.  The twice-smooth specialization improves the rate because
a price mesh of order \(\sqrt{\varepsilon}\) still incurs only
\(O(\varepsilon)\) optimized-value loss, reducing the number of prices
that must be compared.

\section{Experiments}
\label{sec:experiments}

We visualize the finite-horizon cumulative regret of SP-UCB across different
context dimensions and smoothness levels.  The main experiment uses a
controlled demand model in which the demand location varies smoothly with
context and price, while the residual noise follows a uniform distribution.
We run five independent replications for each of six configurations, covering one- and two-dimensional
contexts and smoothness levels from \(\alpha=0.5\) to the Lipschitz case
\(\alpha=1\), at horizons \(T\in\{4000,8000,16000,32000,64000\}\).
The resulting curves are finite-sample diagnostics of the scaling behavior.
Full hyperparameters and implementation details are in Appendix \ref{app:experiment-details}.

\begin{figure}[t]
  \centering
  \includegraphics[width=\textwidth]{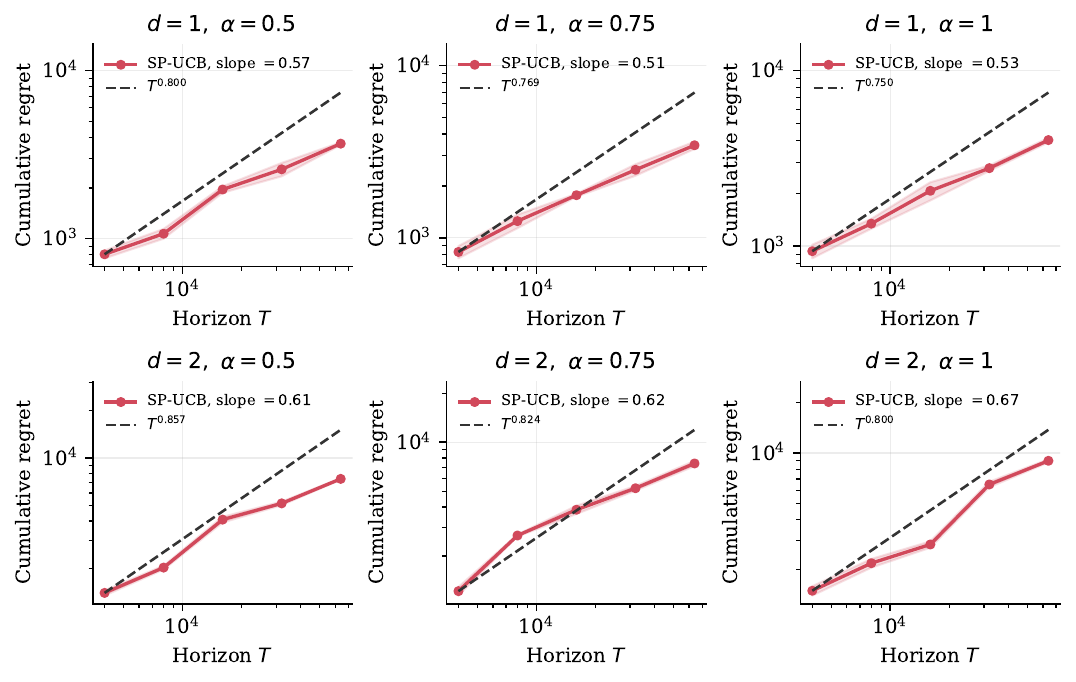}
  \caption{Cumulative-regret scaling of the paper-aligned SP-UCB policy for
  different context dimensions \(d\) and H\"older exponents \(\alpha\).
  Solid curves are means over five exogenous trajectories, and dashed curves give
  representative \(T^{\beta(d,\alpha)}\) guides normalized at the first plotted horizon.}
  \label{fig:scaling-d-alpha}
\end{figure}

The fitted log--log slopes are $(0.57,0.51,0.53)$ for $d=1$ and
$(0.61,0.62,0.67)$ for $d=2$, in increasing order of smoothness.
These fits summarize the five tested horizons and do not estimate a
worst-case asymptotic exponent.
When \(d=2\), profiling occupies a larger fraction of the rounds and the local histories are sparser.  The resulting curves should therefore be read as finite-sample diagnostics of the profiling and UCB interaction, rather than as a verification of the minimax rate.

\paragraph{Additional experiments}
Appendix~\ref{app:experiment-details} compares SP-UCB with neighboring
context-pooled methods on a two-dimensional instance with nonlinear context interactions in demand and a shared
uniform residual distribution.  It also reports the twice-smooth price specialization and a smoothness-mismatch
study that varies the smoothness parameter supplied
to SP-UCB while holding the data-generating environment fixed.

\section{Conclusion}

We studied a seller who must make two decisions at once: how much to charge
and how much to stock.  The seller sees the context and the resulting sales,
but not the demand that was turned away.  We showed that, under a
nonparametric contextual location model, the visible part of censored sales
can still support joint pricing and stocking.  SP-UCB combines inventory
profiling, local contextual pooling, and calibrated price learning, and
achieves matching minimax rates for the stated Lipschitz and twice-smooth
price classes.

The result remains a per-period model without inventory carryover or lead
times, and the shared-noise location structure does not cover arbitrary
price-dependent noise distributions.  Extensions to dynamic inventory states,
weaker feedback, or a more general conditional demand law are natural
directions for future work.

\pagebreak
\clearpage
\bibliographystyle{iclr2027_conference}
\bibliography{references}

\appendix
\numberwithin{equation}{section}

\clearpage

\section{Proof Roadmap and Common Notation}
\label{app:roadmap}

The proofs proceed through identification, estimation, regret accounting,
and hardness.  With only censored sales available to the learner, the
calibrated value provides an estimable target:
\[
  H_x(p):=\ProxyG_x(p)-bq_\rho(x,p)=G_x(p)-bz_\rho(x).
\]
The first equality identifies a target that can be estimated from sales, and
the second shows that its maximizers are exactly those of \(G_x\).  The
upper-bound proof proceeds in three steps.  We first bound the inventory
error and establish quantile visibility.  We then control the calibrated
estimation error and sum the resulting confidence widths.  The supporting
results are organized as follows:
\begin{enumerate}[leftmargin=2em]
  \item \Cref{prop:proxy-identity,lem:calibrated-target} establish that a
  censored observation supplies a bounded proxy sample and a visible
  quantile correction for \(H_x(p)\).
  \item \Cref{lem:profile} profiles only coarse context cells and
  \(\sqrt\varepsilon\)-spaced price anchors, then transfers the nearest anchor
  estimate to each decision price.  The resulting inventory is within
  \(O(\sqrt\varepsilon)\) of the critical inventory.  Its value loss is
  \(O(\varepsilon)\), and its positive separation above \(q_\rho(x,p)\) makes
  the calibration quantile uncensored.
  \item \Cref{lem:kernel-quantile,lem:confidence} give a simultaneous
  confidence interval for \(H_x(p)\) under adaptive price choices.  The
  radius contains a stochastic term \(W_t(x,p)^{-1/2}\) and the two bias
  terms \(\varepsilon\) and \(h^\alpha\).
  \item \Cref{lem:counting} sums the stochastic radii.  Price
  discretization and profiling supply the remaining deterministic regret
  terms, producing the three-term decomposition used in
  \cref{thm:upper,cor:smooth-upper}.
\end{enumerate}

For the lower bound, we work with the admissible subfamily
\(F_x\equiv F_0\) and reveal the common noise law to the learner.  A
value bump \(g(x,p)\) is implemented by the location shift
\(g(x,p)/p\). \Cref{lem:lower-perturbation-stability} verifies that these
shifts remain inside the model class.  The square-root density of the fixed
noise law has finite translation energy, so a location shift of size
\(O(\Delta)\) contributes only \(O(\Delta^2)\) squared Hellinger information.
The complete proof in \cref{app:lower-complete} averages this information over
disjoint price supports inside every context cell and then applies a product
prior over cell labels.  This is the direct-sum step responsible for the
factor \(h^{-d}\).

The common compact support of \(\epsilon\), boundedness of \(\lambda\), and compact action
sets imply constants \(Y_{\max},R_{\max}<\infty\) such that
\(|Y_t|\le Y_{\max}\) and one-round regret is at most \(R_{\max}\) on every
instance considered below.  Constants denoted by \(C,c\) may change from
line to line and depend only on the fixed model constants, never on
\(T,h,\varepsilon\), or the packing labels.

\section{Algorithm Listings}
\label{app:algorithms}

\begin{definition}[Candidate price set]
\label{def:candidate-price-set}
For \(\Delta>0\), let \(\Gamma(\Delta)\) denote the uniform price grid
on \(\cP\) with mesh at most \(\Delta\), including both endpoints.
For a target value accuracy \(\varepsilon>0\), define
\[
  \Gamma:=\Gamma_{\rm dec}=\Gamma(\Delta_{\rm dec}),
  \qquad
  \Delta_{\rm dec}:=
  \begin{cases}
    \varepsilon, & \text{under Lipschitz price-value regularity},\\
    \sqrt{\varepsilon}, & \text{under the smooth-price specialization}.
  \end{cases}
\]
The two cases use the Lipschitz property in \cref{lem:price-lipschitz}
and the additional smoothness in \cref{ass:smooth-price}. Both yield an \(O(\varepsilon)\)
approximation error in optimized expected profit by
\cref{prop:uniform-price-grid}.
\end{definition}

For a context bandwidth \(h\) and target value accuracy
\(\varepsilon=h^\alpha\), we therefore profile on the coarser resolutions
\[
  \delta_x=\varepsilon^{1/(2\alpha)}=\sqrt h,
  \qquad
  \Gamma_y=\Gamma(\sqrt\varepsilon).
\]
Partition the context space into cells
\(\{C_z^{\rm prof}:z\in\mathcal Z_{\rm prof}\}\) of diameter at most
\(\delta_x\), with representative \(z\) and
\(|\mathcal Z_{\rm prof}|\le C \delta_x^{-d}\).  Each cell--anchor pair
\((C_z^{\rm prof},a)\), \(a\in\Gamma_y\), has its own bisection
state.  A decision price \(p\) uses the inventory returned for its nearest
anchor \(\pi(p)\).

We separate the profiling precision from these geometric resolutions by
setting $\eta=c_\eta\sqrt\varepsilon$, where $c_\eta>0$ is a fixed,
horizon-independent input.  The profiler compares empirical stockout rates
with $\tau(a)\pm2\eta$ and uses batches proportional to $\eta^{-2}$, up to
logarithmic factors.  Decreasing $c_\eta$ narrows the stopping band but
requires more samples. It leaves the context partition and anchor grid
unchanged.  The choice $c_\eta=1$ gives the unit-multiplier setting.

Context arrivals determine when these local estimates can be formed.
Because \cref{ass:context} imposes no lower bound on cell probabilities,
requiring all cells to finish profiling would delay price learning in cells
that already have enough observations.  The policy updates the
unfinished profiler of the arriving cell and begins UCB immediately upon that
cell's completion.  The persistent states and one-round update are given in
\cref{alg:profiler}.

Under the certified resolution \(h\le h_{\rm vis}\), \cref{lem:profile}
guarantees, with high probability, that every completed cell returns
inventories with \(O(\sqrt\varepsilon)\) error, \(O(\varepsilon)\) value
loss, and a visible \(\rho\)-quantile.  The total number of profiling rounds
is deterministically bounded by
\(\widetilde O(h^{-d/2}\varepsilon^{-3/2})\).  These guarantees allow each
completed cell to use subsequent sales for both terms of \(H_x(p)\).

\begin{algorithm}[H]
\caption{Surrogate-Profit UCB (SP-UCB)}
\label{alg:mck}
\begin{algorithmic}[1]
\Require horizon \(T\), context smoothness \(\alpha\), visible quantile level
  \(\rho\), price mode \(s\in\{\mathrm{Lip},\mathrm{smooth}\}\), fixed tuning
  constants \(C_{\rm tune},h_{\rm vis},c_\eta>0\)
\State Fix \(a_{\rm fb}=(p_{\min},y_{\min})\) and set \(\delta=T^{-2}\)
\If{\(s=\mathrm{Lip}\)}
  \State Set \(h=(\log(eT)/T)^{1/(d+3\alpha)}\),
  \(\varepsilon=h^\alpha\), and \(\Delta_{\rm dec}=\varepsilon\)
\Else
  \State Set \(h=(\log(eT)/T)^{2/(2d+5\alpha)}\),
  \(\varepsilon=h^\alpha\), and \(\Delta_{\rm dec}=\sqrt{\varepsilon}\)
\EndIf
\State Set \(\Gamma_{\rm dec}=\Gamma(\Delta_{\rm dec})\),
\(\Gamma_y=\Gamma(\sqrt\varepsilon)\), and
\(\delta_x=\varepsilon^{1/(2\alpha)}\). Construct
\(\mathcal Z_{\rm prof}\) with cell diameter at most \(\delta_x\), and
initialize the persistent anchor-profiler states in \cref{alg:profiler}
with \(\eta=c_\eta\sqrt\varepsilon\)
\If{\(h>h_{\rm vis}\)}
  \State Play \(a_{\rm fb}\) on all \(T\) rounds and \textbf{stop}
\EndIf
\For{round \(t=1,\ldots,T\)}
  \State Observe context \(x_t\) and let \(z=z_{\rm prof}(x_t)\)
  \If{cell \(z\) is not profiled}
    \State Execute the one-round update in \cref{alg:profiler} and
    \textbf{continue}
  \EndIf
  \State Compute \(U_t(x_t,p)=+\infty\) if
  \(W_t(x_t,p)=0\), and otherwise use
  \[
    U_t(x_t,p)=\widehat G_t^{\rm cal}(x_t,p)
    +C_{\rm tune}\sqrt{\frac{\log(T|\Gamma_{\rm dec}|/h)}
    {\max\{1,W_t(x_t,p)\}}}
    +C_{\rm tune}(\varepsilon+h^\alpha)
  \]
  \State Choose \(p_t\in\argmax_{p\in\Gamma_{\rm dec}}U_t(x_t,p)\), let
  \(a_t=\pi(p_t)\) and \(y_t=\widetilde y_z(a_t)\), play
  \((p_t,y_t)\), and observe \(O_t\)
  \State Append \((x_t,O_t,Y_t)\), where
  \(Y_t=(p_t+b)O_t-h_c(y_t-O_t)\),
  to the history of \(p_t\)
\EndFor
\end{algorithmic}
\end{algorithm}

The startup fallback covers horizons for which \(h>h_{\rm vis}\), before
quantile visibility is certified.  For all sufficiently large \(T\), the
profiling--UCB policy is used. The finite remaining range of horizons is
absorbed into the theorem constant.  On UCB rounds, summing local confidence
widths along the realized trajectory gives
\(\widetilde O(\sqrt{T|\Gamma|h^{-d}})\).  Together with the profiling cost
and approximation losses above, this leads to the regret bounds in
\cref{sec:results}.  The concentration and counting arguments are given in
\cref{app:kernel-concentration}.

\begin{algorithm}[H]
\caption{Inventory Search}
\label{alg:profiler}
\begin{algorithmic}[1]
\Require horizon \(T\), coarse profiler partition
\(\mathcal Z_{\rm prof}\), current context \(x_t\), its cell
\(z=z_{\rm prof}(x_t)\), persistent profiler state for \(z\), value accuracy
\(\varepsilon\), anchor grid \(\Gamma_y\), confidence level \(\delta\),
  fixed tuning constants \(C_{\rm tune}\) and \(c_\eta>0\)
\State At initialization, set \(\eta=c_\eta\sqrt{\varepsilon}\) and
\[
  B_y=1+\left\lceil
    \log_2\!\left(1+\frac{y_{\max}-y_{\min}}
    {C_{\rm tune}\eta}\right)
  \right\rceil,
  \qquad
  m_y=\left\lceil
    C_{\rm tune}\eta^{-2}
    \log(T|\Gamma_y||\mathcal Z_{\rm prof}|B_y/\delta)
  \right\rceil
\]
\State At initialization, for every
\((z,a)\in\mathcal Z_{\rm prof}\times\Gamma_y\), initialize
\([L_{z,a},U_{z,a}]=[y_{\min},y_{\max}]\)
\State Select the next unfinished anchor \(a\) in the round-robin schedule of
cell \(z\)
\State Set \(p_t=a\) and \(y_t=(L_{z,a}+U_{z,a})/2\), play
\((p_t,y_t)\), observe \(O_t\), and append
\(I_t=\one\{O_t\ge y_t\}\) to the current batch of \((z,a)\)
\If{the current batch contains \(m_y\) indicators}
  \State Let \(\widehat S\) be their average
  \If{\(\widehat S>\tau(a)+2\eta\)}
    \State Set \(L_{z,a}=y_t\)
  \ElsIf{\(\widehat S<\tau(a)-2\eta\)}
    \State Set \(U_{z,a}=y_t\)
  \Else
    \State Mark \((z,a)\) finished and set \(\widetilde y_z(a)=y_t\)
  \EndIf
  \If{\((z,a)\) is unfinished and
  \(U_{z,a}-L_{z,a}\le C_{\rm tune}\eta\)}
    \State Mark \((z,a)\) finished and set
    \(\widetilde y_z(a)=(L_{z,a}+U_{z,a})/2\)
  \EndIf
  \State Clear the completed batch
\EndIf
\If{all \((z,a')\), \(a'\in\Gamma_y\), are finished}
  \State Mark cell \(z\) profiled and define
  \(\widetilde y_x(p)=\widetilde y_z(\pi(p))\) for
  \(x\in C_z^{\rm prof}\), \(p\in\Gamma_{\rm dec}\), where
  \(\pi(p)\in\argmin_{a\in\Gamma_y}|p-a|\)
\EndIf
\end{algorithmic}
\end{algorithm}

\section{Formal Statements and Tuning}
\label{app:formal-results}

\subsection{Tuning Convention}
\label{app:tuning}

The policy inputs are the horizon \(T\), context smoothness exponent
\(\alpha\), calibration level \(\rho\), price-regularity mode, and the
fixed tuning constants \(C_{\rm tune},h_{\rm vis},c_\eta\). The action
ranges and costs are given, and the kernel is chosen as part of the policy.
For the theoretical guarantees, tuning uses known numerical bounds for
the model class. These include bounds on
\[
  L_x,L_p,L_F,L_{\rm loc},K,C_y,Y_{\max},R_{\max},
  c_y^{-1},\kappa^{-1},r_y^{-1},r_q^{-1},m_q^{-1}
\]
and fixed action/kernel constants. Here \(K:=\kappa'\) is the density
upper bound, and \(Y_{\max},R_{\max}\) bound proxy magnitude and
one-period regret. Bounds for derived quantities are obtained from the
primitive model bounds. The smooth-price class also fixes a bound
\(L_{G,2}\). Fix a finite multiplier
\(c_\eta>0\), independent of \(T\), and set
\(\eta=c_\eta\sqrt{\varepsilon}\).  The choice \(c_\eta=1\)
recovers the unit-multiplier policy.  Choose \(C_{\rm tune}\)
large enough for the stockout concentration bound and the stochastic and
context-bias constants \(C_{\rm ker}\) in \cref{lem:confidence}, including
the fixed factor comparing the logarithms in that lemma and the index.
These requirements depend only on the numerical envelopes, not on the
inventory approximation error produced by the stopping rule.

To account for that error separately, define
\[
\begin{aligned}
  L_\phi&:=\sup_{p\in\cP}|\phi'(p)|
       =\frac{h_c}{(p_{\min}+b+h_c)^2},\\
  C_{\rm prof}&:=1+L_{\rm loc}+L_p+L_\phi/\kappa
        +c_\eta\max\{C_{\rm tune}/2,3/\kappa\},\\
  B_{\rm prof}&:=C_y C_{\rm prof}^2.
\end{aligned}
\]
After fixing \(C_{\rm tune}\) and \(c_\eta\), choose \(0<h_{\rm vis}\le1\) so small
that, for every \(h\le h_{\rm vis}\),
\[
\begin{aligned}
  C_{\rm prof}\sqrt{h^\alpha}
  &\le \min\{r_y,r_q/4,m_q/2\},\\
  L_F h^\alpha&\le\kappa r_q/2,\\
  (L_x+L_F/\kappa)h^\alpha&\le\min\{m_q/8,r_q/4\},\\
  C_{\rm tune}h^\alpha&\le\tfrac14\sqrt{h^\alpha}.
\end{aligned}
\]
Such a choice exists because \(\alpha>0\).  The constants
\(C_{\rm prof},B_{\rm prof}\) quantify the inventory approximation error
in the proof. The upper-bound analysis retains this error through a common
analytical shift of the indices. Fixed admissible tuning choices affect
multiplicative constants and finite-horizon transients. The regret exponent
is unchanged. Upper-bound constants may depend on \(c_\eta\) and its
inverse. The demand surface and residual laws remain unknown.

\subsection{Auxiliary Statements}
\label{app:auxiliary-statements}

\begin{lemma}[Unique optimal inventory]
\label{lem:unique-inventory}
Under \cref{ass:context,ass:inventory}, for every context \(x\) and
price \(p\), \(y_x^*(p)=\lambda(x,p)+z_p(x)\) is the unique maximizer of
\(Q_x(p,\cdot)\) over \([y_{\min},y_{\max}]\).  It is also the unique
solution of \(S_x(y\mid p)=\tau(p)\).
\end{lemma}

\begin{lemma}[Quantile separation]
\label{lem:quantile-separation}
Under \cref{ass:context,ass:inventory}, the calibration and critical
quantiles satisfy, uniformly in \(x,p\),
\[
  y_x^*(p)-q_\rho(x,p)=z_p(x)-z_\rho(x)
  \ge
  \frac{\inf_{q\in\cP}\phi(q)-\rho}{\kappa'}
  >0 .
\]
Throughout the analysis, \(m_q\) denotes a fixed positive constant
smaller than this lower bound.
\end{lemma}

\begin{lemma}[Lipschitz price values]
\label{lem:price-lipschitz}
Under \cref{ass:context,ass:holder}, there is a finite \(L_G\), depending
only on the model bounds, such that
\[
  |G_x(p)-G_x(q)|\le L_G|p-q|
  \qquad\text{for all }x\in\cX,\ p,q\in\cP.
\]
\end{lemma}

\begin{proposition}[Location-model separation]
\label{prop:location-separation}
Under \cref{ass:context,ass:holder,ass:inventory}, \(y_x^*(p)\) is the
unique maximizer of \(Q_x(p,\cdot)\), and
\[
  G_x(p)=p\lambda(x,p)+C_x(p),
\]
where
\[
  C_x(p)
  =
  p\E_{F_x}[\min\{\epsilon,z_p(x)\}]
  -b\E_{F_x}[(\epsilon-z_p(x))^+]
  -h_c\E_{F_x}[(z_p(x)-\epsilon)^+].
\]
Moreover, there is a fixed \(L_{\rm loc}<\infty\) such that, uniformly in
\(p\), each of \(y_x^*(p)\), \(q_\rho(x,p)\), \(G_x(p)\), and
\(\ProxyG_x(p)\) changes by at most
\(L_{\rm loc}\norm{x-x'}^\alpha\) between contexts \(x,x'\).
\end{proposition}

\begin{proposition}[Uniform price discretization]
\label{prop:uniform-price-grid}
Under \cref{ass:context,ass:holder},
\[
  |\Gamma(\varepsilon)|
  \le
  2+\frac{p_{\max}-p_{\min}}{\varepsilon}
\]
and
\[
  \sup_x
  \left\{
    \max_{p\in\cP}G_x(p)
    -
    \max_{q\in\Gamma(\varepsilon)}G_x(q)
  \right\}
  \le
  L_G\varepsilon.
\]
Under \cref{ass:smooth-price}, the grid
\(\Gamma(\sqrt{\varepsilon})\) has cardinality \(O(\varepsilon^{-1/2})\)
and approximation error \(O(\varepsilon)\).
\end{proposition}

\begin{proposition}[Preservation of inventory profit gaps]
\label{prop:proxy-identity}
The observable surrogate
preserves all expected profit differences between inventory decisions at a
fixed context and price.  Specifically, for every \(x,p\) and any feasible
\(y,y'\),
\[
  \ProxyQ_x(p,y)-\ProxyQ_x(p,y')
  =Q_x(p,y)-Q_x(p,y').
\]
Consequently, it preserves both the optimal inventory and the loss from
choosing a suboptimal inventory:
\[
  \argmax_y\ProxyQ_x(p,y)=\argmax_yQ_x(p,y),
  \qquad
  \ProxyG_x(p)-\ProxyQ_x(p,y)=G_x(p)-Q_x(p,y).
\]
\end{proposition}

\begin{lemma}[Mean calibration up to a within-context common shift]
\label{lem:calibrated-target}
Under \cref{ass:context}, define
\(q_\rho(x,p)=\lambda(x,p)+z_\rho(x)\).  Then
\[
  \ProxyG_x(p)-bq_\rho(x,p)
  =
  G_x(p)-bz_\rho(x).
\]
For a fixed \(x\), the rightmost offset is independent of \(p\) and \(y\).
Hence an upper confidence bound for the left-hand side orders prices exactly
as an upper confidence bound for \(G_x(p)\).
\end{lemma}

\begin{lemma}[Two-resolution asynchronous profiling guarantee]
\label{lem:profile}
Under
\cref{ass:context,ass:holder,ass:inventory}, take
\(\varepsilon=h^\alpha\) and \(\eta=c_\eta\sqrt{\varepsilon}\)
for a fixed \(c_\eta>0\), with tuning as in \cref{app:tuning}.
The constant \(C\) below may depend on \(c_\eta\).  Let \(N_{\rm prof}(T)\) be the number of rounds up
to horizon \(T\) on which the asynchronous policy performs a profiling
update.  Deterministically,
\[
  N_{\rm prof}(T)
  \le
  C|\Gamma_y|\delta_x^{-d}\varepsilon^{-1}
  \log^2\!\left(
    \frac{CT|\Gamma_y|\delta_x^{-d}}{\delta}
  \right)
\]
and, for \(h\le h_{\rm vis}\), with probability at least \(1-\delta\),
every completed cell has returned estimates satisfying, uniformly in
\(p\in\Gamma_{\rm dec}\) and \(x\) in that cell,
\[
  |\widetilde y_x(p)-y_x^*(p)|
  \le C_{\rm prof}\sqrt{\varepsilon}
\]
and
\[
  G_x(p)-Q_x(p,\widetilde y_x(p))
  =
  \ProxyG_x(p)-\ProxyQ_x(p,\widetilde y_x(p))
  \le B_{\rm prof}\varepsilon.
\]
Under the certified resolution \(h\le h_{\rm vis}\),
\[
  \widetilde y_x(p)\ge q_\rho(x,p)+m_q/2,
\]
so the \(\rho\)-quantile of
\(\min\{D,\widetilde y_x(p)\}\) equals \(q_\rho(x,p)\).  No completion
guarantee for cells that have not received enough arrivals is needed.  Since
\(|\Gamma_y|=O(\varepsilon^{-1/2})\) and
\(\delta_x^{-d}=h^{-d/2}\),
\[
  N_{\rm prof}(T)
  \le \widetilde O\!\left(h^{-d/2}\varepsilon^{-3/2}\right).
\]
\end{lemma}

\subsection{Upper and Lower Bounds}

\begin{theorem}[Lipschitz-price contextual upper bound]
\label{thm:upper}
Under
\cref{ass:context,ass:holder,ass:inventory},
with \(s=\mathrm{Lip}\), and with the certified tuning inputs in
\cref{app:tuning}, \(\mathsf{SP}\)-UCB satisfies
\[
  \Reg_T
  \le
  C
  T^{\frac{d+2\alpha}
          {d+3\alpha}}
  \polylog(T).
\]
\end{theorem}

\begin{corollary}[Twice-smooth price upper bound]
\label{cor:smooth-upper}
Under \cref{ass:context,ass:holder,ass:inventory,ass:smooth-price},
invoke \(\mathsf{SP}\)-UCB with
\(s=\mathrm{smooth}\), using the same fixed tuning constants and fallback
rules.  Its price mesh is \(\Delta_p=\sqrt{\varepsilon}\), with
\[
  h=(\log(eT)/T)^{2/(2d+5\alpha)},
  \qquad
  \varepsilon=h^\alpha.
\]
Then
\[
  \Reg_T
  \le
  C
  T^{\frac{2d+3\alpha}{2d+5\alpha}}
  \polylog(T).
\]
\end{corollary}

\begin{theorem}[Context-conditioned shared-noise lower bounds]
\label{thm:lower}
Assume \(b>0\) and \(h_c>0\), and let \(0<\alpha\le1\).
There exist fixed nondegenerate compact price and inventory ranges and a
context-conditioned shared-noise location subclass satisfying
\cref{ass:context,ass:holder,ass:inventory}
such that every policy obeys, under Lipschitz price regularity,
\[
  \sup_{\cI}\Reg_T(\pi;\cI)
  \ge
  c
  T^{\frac{d+2\alpha}
          {d+3\alpha}}.
\]
For the subclass also satisfying \cref{ass:smooth-price}, the lower bound is
\[
  cT^{\frac{2d+3\alpha}{2d+5\alpha}}.
\]
\end{theorem}

\begin{corollary}[Context-conditioned shared-noise minimax rates]
\label{cor:minimax}
Fix common numerical envelopes for all constants in
\cref{ass:context,ass:holder,ass:inventory}
and the fixed tuning constants, as well as fixed action ranges containing the
subclass in \cref{thm:lower}.  Let \(\mathfrak I_{\rm Lip}\) be the resulting
context-conditioned shared-noise class with Lipschitz price values, and let
\(\mathfrak I_{\rm sm}\) additionally satisfy \cref{ass:smooth-price}
with a fixed bound \(L_{G,2}\). For fixed
\(b,h_c>0\), the upper and lower bounds yield
\[
  \inf_\pi\sup_{\cI\in\mathfrak I_{\rm Lip}}\Reg_T(\pi;\cI)
  =
  \widetilde\Theta\!\left(
    T^{\frac{d+2\alpha}
            {d+3\alpha}}
  \right).
\]
Under twice-smooth price values the corresponding rate is
\[
  \inf_\pi\sup_{\cI\in\mathfrak I_{\rm sm}}\Reg_T(\pi;\cI)
  =
  \widetilde\Theta\!\left(
      T^{\frac{2d+3\alpha}{2d+5\alpha}}
    \right).
\]
\end{corollary}

\section{Proofs for the Proxy and Profiling Results}

\subsection{Profit Identities and Inventory Properties}

\begin{proof}[Proof of \cref{prop:proxy-identity}]
The elementary identity
\[
  D=\min\{D,y\}+(D-y)^+
\]
gives
\begin{align*}
  r((p,y),D)+bD
  &=
  pO-b(D-y)^+-h_c(y-D)^+
  +b\{O+(D-y)^+\}\\
  &=
  (p+b)O-h_c(y-O),
\end{align*}
because \(y-O=(y-D)^+\).  This is exactly \(Y(p,y,O)\).
Taking conditional expectations and using mean-zero noise yields
\[
  \ProxyQ_x(p,y)=Q_x(p,y)+b\lambda(x,p).
\]
For fixed \(x,p\), the added term is independent of \(y\). Subtracting
the identities at \(y\) and \(y'\) proves the claimed equality of profit
differences and hence the equality of maximizers.  Maximizing over inventory
also gives
\[
  \ProxyG_x(p)=G_x(p)+b\lambda(x,p).
\]
Subtracting the preceding identity for \(\ProxyQ_x(p,y)\) proves the
equality of suboptimality losses.
\end{proof}

\begin{proof}[Proof of \cref{lem:calibrated-target}]
The location family gives \(q_\rho(x,p)=\lambda(x,p)+z_\rho(x)\).  Substitute this identity into
\(\ProxyG_x(p)=G_x(p)+b\lambda(x,p)\).
\end{proof}

\begin{proof}[Proof of \cref{lem:unique-inventory}]
Fix \(x,p\), abbreviate \(\lambda=\lambda(x,p)\), and write
\(y_x^*=y_x^*(p)=\lambda+z_p(x)\).  Since
\(F_x(z_p(x))=\phi(p)\),
\[
  S_x(y_x^*\mid p)=1-F_x(z_p(x))=\tau(p).
\]
For \(y<y_x^*\), monotonicity gives
\(S_x(y\mid p)\ge\tau(p)\), while for \(y>y_x^*\) it gives
\(S_x(y\mid p)\le\tau(p)\).  The lower density bound around \(z_p(x)\)
makes these inequalities strict away from \(y_x^*\), so the root is unique.
Because \(y_x^*(p)\) is feasible with a margin by \cref{ass:inventory}, the
first-order identity in \cref{sec:algorithm}, together with the
inventory-independent surrogate offset, implies that it is also the unique
maximizer of \(Q_x(p,\cdot)\).
\end{proof}

\begin{proof}[Proof of \cref{lem:quantile-separation}]
By \cref{lem:unique-inventory} and the definition of the calibration
quantile, \(y_x^*(p)-q_\rho(x,p)=z_p(x)-z_\rho(x)\).
The density upper bound gives
\[
  \phi(p)-\rho
  =\int_{z_\rho(x)}^{z_p(x)}f_x(u)\,du
  \le\kappa'\{z_p(x)-z_\rho(x)\}.
\]
Taking the infimum of \(\phi(p)\) over the price range proves the claim.
\end{proof}

\begin{proof}[Proof of \cref{lem:inventory-consequences}]
Fix \(x,p\), and write \(\lambda=\lambda(x,p)\) and
\(y_x^*=y_x^*(p)=\lambda+z_p(x)\).
For \(|y-y_x^*|\le r_y\le r_q\),
\[
\begin{aligned}
  |S_x(y\mid p)-\tau(p)|
  &=
  |F_x(y-\lambda)-F_x(z_p(x))|  \\
  &=
  \left|\int_{z_p(x)}^{y-\lambda}f_x(u)\,du\right|.
\end{aligned}
\]
The density bounds in \cref{ass:inventory} therefore give
\[
  \kappa|y-y_x^*|
  \le
  |S_x(y\mid p)-\tau(p)|
  \le
  K|y-y_x^*|.
\]
Finally, integrating
\(\partial_yQ_x(p,y)=(p+b+h_c)\{S_x(y\mid p)-\tau(p)\}\) between \(y\)
and \(y_x^*\) yields
\[
  0\le
  G_x(p)-Q_x(p,y)
  \le
  \frac{p_{\max}+b+h_c}{2}K|y-y_x^*|^2 .
\]
The claim follows after taking \(c_y=\kappa\) and increasing \(C_y\) if
necessary.
\end{proof}

\subsection{Regularity and Approximation}

\begin{proof}[Proof of \cref{prop:location-separation}]
Fix \(x,p\), abbreviate \(\lambda=\lambda(x,p)\), and write
\(y_x^*(p)=\lambda+z_p(x)\).  By \cref{lem:unique-inventory}, this
inventory is the unique maximizer of \(Q_x(p,\cdot)\).

At this inventory, the following identities hold sample by sample:
\[
\begin{aligned}
  \min\{\lambda(x,p)+\epsilon,\lambda(x,p)+z_p(x)\}
  &=\lambda(x,p)+\min\{\epsilon,z_p(x)\},\\
  (\lambda(x,p)+\epsilon-\lambda(x,p)-z_p(x))^+
  &=(\epsilon-z_p(x))^+,\\
  (\lambda(x,p)+z_p(x)-\lambda(x,p)-\epsilon)^+
  &=(z_p(x)-\epsilon)^+.
\end{aligned}
\]
Substitution into the definition of \(Q_x\) gives
\[
\begin{aligned}
  G_x(p)
  &=p\lambda(x,p)
    +p\E_{F_x}[\min\{\epsilon,z_p(x)\}]
    -b\E_{F_x}[(\epsilon-z_p(x))^+]
    -h_c\E_{F_x}[(z_p(x)-\epsilon)^+]\\
  &=p\lambda(x,p)+C_x(p).
\end{aligned}
\]

It remains to prove the asserted context regularity.  Put
\(d_{xx'}=\norm{x-x'}\).  For
\(u\in\{\rho\}\cup\{\phi(p):p\in\cP\}\), put
\(\delta_{xx'}=L_Fd_{xx'}^\alpha\).  When
\(\delta_{xx'}<\kappa r_q\), the two quantiles must be within \(r_q\): were,
say, \(z_u(x')>z_u(x)+r_q\), integrating \(f_x\ge\kappa\) from
\(z_u(x)\) to \(z_u(x)+r_q\) would give
\(F_{x'}(z_u(x)+r_q)\ge
F_x(z_u(x)+r_q)-\delta_{xx'}>u\), contradicting the definition of
\(z_u(x')\).
Inside that neighborhood, the same integration argument gives
\[
  \kappa|z_u(x)-z_u(x')|\le\delta_{xx'}.
\]
When \(\delta_{xx'}\ge\kappa r_q\), compact support gives
\[
  |z_u(x)-z_u(x')|
  \le \overline z-\underline z
  \le \frac{(\overline z-\underline z)L_F}{\kappa r_q}
       d_{xx'}^\alpha.
\]
Thus
both \(z_\rho(x)\) and \(z_p(x)\) are uniformly H\"older in \(x\), and the
same is true of \(q_\rho(x,p)\) and \(y_x^*(p)\) after adding the location
bound from \cref{ass:holder}.

For completeness, let \(P_{x,p}\) denote the law of
\(\lambda(x,p)+\epsilon\).  In one dimension,
\[
\begin{aligned}
  W_1(P_{x,p},P_{x',p})
  &\le |\lambda(x,p)-\lambda(x',p)|+W_1(F_x,F_{x'})\\
  &\le \{L_x+(\overline z-\underline z)L_F\}d_{xx'}^\alpha,
\end{aligned}
\]
where the last step uses
\(W_1(F_x,F_{x'})=\int_{\R}|F_x(u)-F_{x'}(u)|\,du\).
For fixed \((p,y)\), the profit is Lipschitz in demand with constant at most
\(L_r:=\max\{p_{\max}+h_c,b\}\).  The Kantorovich--Rubinstein inequality
therefore bounds
\(|Q_x(p,y)-Q_{x'}(p,y)|\) by a constant times
\(d_{xx'}^\alpha\), uniformly in \(y\).  Taking maxima over the common
inventory interval proves the same bound for \(G_x(p)\).  Finally,
\(\ProxyG_x(p)=G_x(p)+b\lambda(x,p)\), so \(\ProxyG_x(p)\) is H\"older as
well.  Choosing one constant that dominates these four bounds proves the last
claim.
\end{proof}

\begin{proof}[Proof of \cref{lem:price-lipschitz}]
Fix a context \(x\), and couple demand at prices \(p,q\) using the same
residual \(\epsilon\sim F_x\). Then
\(D_p-D_q=\lambda(x,p)-\lambda(x,q)\).
Let \(M\) be a common bound on demand and the absolute inventory levels,
which exists by \cref{ass:context} and the compact action set.
For fixed price and inventory, profit is Lipschitz in demand with constant
\(L_r=\max\{p_{\max}+h_c,b\}\). At a fixed demand realization, its
price coefficient is \(\min\{D,y\}\), whose absolute value is at most
\(M\). Therefore, for every feasible \(y\),
\[
\begin{aligned}
  |Q_x(p,y)-Q_x(q,y)|
  &\le M|p-q|+L_r|\lambda(x,p)-\lambda(x,q)|\\
  &\le (M+L_rL_p)|p-q|.
\end{aligned}
\]
Taking maxima over the common inventory interval yields the claimed bound
with \(L_G=M+L_rL_p\).
\end{proof}

\begin{proof}[Proof of \cref{prop:uniform-price-grid}]
The displayed cardinality bound follows directly from the construction of
\(\Gamma(\varepsilon)\). Fix a maximizer
\(p_x^*\in\argmax_{p\in\cP}G_x(p)\). It has a grid point within
distance \(\varepsilon\); the Lipschitz bound then gives
\[
  G_x(p_x^*)-\max_{q\in\Gamma(\varepsilon)}G_x(q)
  \le
  L_G\varepsilon.
\]
Under the bounded-second-derivative condition, Taylor's theorem around an
interior maximizer gives a loss at most
\(\tfrac12L_{G,2}|p-p_x^*|^2\).  The same order follows at a boundary
maximizer because both endpoints themselves belong to the grid, and hence
the discretization loss is then zero.  A grid of mesh \(\sqrt{\varepsilon}\)
therefore incurs \(O(\varepsilon)\) loss and contains
\(O(\varepsilon^{-1/2})\) points.
\end{proof}

\begin{lemma}[Within-cell approximation]
\label{lem:context-grid}
For a sufficiently fine context grid, the inventory
\(\overline y_C(p)\) maximizing the cell-average expected surrogate score
is unique and satisfies
\[
  \sup_{x\in C}|\overline y_C(p)-y_x^*(p)|
  \le L_{\rm loc}\delta_x^{\alpha}
\]
for every cell \(C\) of positive probability and every price \(p\).
\end{lemma}

\paragraph{Context pooling and estimation error.}
The resolution condition in \cref{lem:context-grid} is satisfied by
\(L_{\rm loc}\delta_x^\alpha\le r_q/2\).  The cell inventory target is
characterized by
\[
  \overline S_C(\overline y_C(p)\mid p)=\tau(p),
  \qquad
  \overline S_C(y\mid p)
  :=\E[S_X(y\mid p)\mid X\in C].
\]
Any feasible estimate \(\widehat y_C(p)\) satisfying
\(|\widehat y_C(p)-\overline y_C(p)|\le e\) obeys
\[
  \sup_{x\in C}|\widehat y_C(p)-y_x^*(p)|
  \le e+L_{\rm loc}\delta_x^\alpha.
\]
Whenever \(e+L_{\rm loc}\delta_x^\alpha\le r_y\), it also satisfies
\[
  \sup_{x\in C}
  \bigl(G_x(p)-Q_x(p,\widehat y_C(p))\bigr)
  \le C_y\bigl(e+L_{\rm loc}\delta_x^\alpha\bigr)^2.
\]

\begin{proof}[Proof of \cref{lem:context-grid} and the estimation bounds]
Fix a positive-probability cell \(C\) and a price \(p\).
By \cref{prop:location-separation},
\[
  |y_x^*(p)-y_{x'}^*(p)|
  \le L_{\rm loc}\delta_x^\alpha
  \qquad (x,x'\in C).
\]
Put \(\ell=\inf_{x\in C}y_x^*(p)\) and
\(u=\sup_{x\in C}y_x^*(p)\).  The inventory range assumption
makes both endpoints feasible.  Continuity and monotonicity give
\[
  \overline S_C(\ell\mid p)\ge\tau(p),
  \qquad
  \overline S_C(u\mid p)\le\tau(p),
\]
so there is a root \(\overline y_C(p)\in[\ell,u]\).
Since \(u-\ell\le L_{\rm loc}\delta_x^\alpha\le r_q/2\), every
point within \(r_q/2\) of this root is within \(r_q\) of each
\(y_x^*(p)\).  The density lower bound in \cref{ass:inventory}
therefore makes \(\overline S_C(\cdot\mid p)\) strictly decreasing
in this neighborhood.  Monotonicity then ensures global uniqueness.

The derivative of the cell-average expected surrogate score is
\[
  \partial_y\E[\ProxyQ_X(p,y)\mid X\in C]
  =(p+b+h_c)\bigl(\overline S_C(y\mid p)-\tau(p)\bigr).
\]
Differentiation under the expectation is justified by the uniform bound
on the inventory derivative.  Its sign is positive below the unique root
and negative above it, so the root uniquely maximizes this score.
The bound \(u-\ell\le L_{\rm loc}\delta_x^\alpha\) proves
\cref{lem:context-grid}.  The estimation bound follows by the triangle
inequality, and the profit-loss bound follows from
\cref{lem:inventory-consequences}.
\end{proof}

\subsection{Inventory Profiling Guarantees}

\begin{proof}[Proof of \cref{lem:profile}]
Put \(\eta=c_\eta\sqrt\varepsilon\) for fixed \(c_\eta>0\), and
\(\delta_x=\varepsilon^{1/(2\alpha)}\).  We first construct a single event
on which every anchor-price bisection decision is correct.  Fix a profiler
cell \(C_z^{\rm prof}\) with positive context probability and an anchor
\(a\in\Gamma_y\).  Define its conditional-average stockout curve
\[
  \overline S_z(y\mid a)
  :=\E[S_{x_t}(y\mid a)\mid x_t\in C_z^{\rm prof}].
\]
Cells of zero probability receive no arrivals and contribute no profiling
rounds.  Every individual root in the cell equals
\(y_x^*(a)=\lambda(x,a)+z_a(x)\), and therefore
\[
  \sup_{x,x'\in C_z^{\rm prof}}
  |y_x^*(a)-y_{x'}^*(a)|
  \le L_{\rm loc} \delta_x^\alpha
  =L_{\rm loc}\sqrt\varepsilon.
\]
Continuity and monotonicity place a root
\(\overline y_z^*(a)\) of
\(\overline S_z(\cdot\mid a)=\tau(a)\) between the smallest and largest
individual roots.  Consequently
\begin{equation}
  \sup_{x\in C_z^{\rm prof}}
  |\overline y_z^*(a)-y_x^*(a)|
  \le L_{\rm loc}\sqrt\varepsilon.
\label{eq:profile-context-bias}
\end{equation}
By the tuning conditions, \(L_{\rm loc}\sqrt\varepsilon\le r_q/4\).  Every threshold
within \(r_q/2\) of \(\overline y_z^*(a)\) is therefore within
\(3r_q/4\) of every individual root in the cell.  Integrating the density
bounds and then averaging shows that, for any two such thresholds
\(u<v\),
\[
  \kappa(v-u)\le\overline S_z(u\mid a)-\overline S_z(v\mid a)
  \le K(v-u).
\]
Thus the averaged root is unique there. Monotonicity makes it unique globally.

The within-cell scheduler uses only the cell index and the past to schedule each update.  Conditional on an
arrival to the cell, the context therefore retains the fixed conditional law
used in \(\overline S_z\).  During a fixed threshold batch, each stockout
indicator is a bounded martingale difference after subtracting
\(\overline S_z(y\mid a)\).  Hoeffding's inequality gives stochastic error at
most \(\eta\) from
\[
  m_y\ge C\eta^{-2}
  \log\!\left(
    \frac{|\Gamma_y||\mathcal Z_{\rm prof}|B_y}{\delta}
  \right)
\]
samples, where
\[
  B_y
  :=1+\left\lceil
    \log_2\!\left(1+
      \frac{y_{\max}-y_{\min}}{C_{\rm tune}\eta}
    \right)
  \right\rceil
\]
bounds the number of thresholds queried by one bisection.  Taking the union
bound over the at most
\(|\Gamma_y||\mathcal Z_{\rm prof}|B_y\) adaptive batches yields
an event \(\mathcal E_{\rm bis}\) of probability at least \(1-\delta/2\) on
which
\begin{equation}
  |\widehat S-\overline S_z(y\mid a)|\le\eta
\label{eq:profile-uniform-event}
\end{equation}
for every queried batch.  The union bound is valid for adaptive thresholds
because the conditional tail bound holds at every batch given its starting
history.  A batch unfinished at \(T\) can be continued on subsequent cell
arrivals for this argument. Positive cell probability ensures that the
extended batch finishes almost surely.  Bounding the extended batch also
bounds the event that an actually completed batch is inaccurate, without
conditioning on completion by \(T\).

We next verify the bisection invariant.  Initially
\(\overline y_z^*(a)\in[L_{z,a},U_{z,a}]\).  The averaged stockout curve is
nonincreasing.  On
\(\mathcal E_{\rm bis}\),
\(\widehat S>\tau(a)+2\eta\) implies
\(\overline S_z(y\mid a)>\tau(a)+\eta\), hence
\(y<\overline y_z^*(a)\) and replacing the lower endpoint by \(y\) preserves
the invariant.  The symmetric statement holds when
\(\widehat S<\tau(a)-2\eta\).  In the remaining case,
\[
  |\overline S_z(y\mid a)-\tau(a)|\le3\eta.
\]
The choice of \(C_{\rm prof}\) and \(h_{\rm vis}\) gives
\(3\eta\le\kappa r_q/4\).  A threshold at distance at least \(r_q/2\)
from the averaged root has stockout gap at least \(\kappa r_q/2\), by
the preceding slope bound and monotonicity.  Hence the ambiguous branch
terminates inside that neighborhood, where the same bound gives
\begin{equation}
  |y-\overline y_z^*(a)|\le 3\eta/\kappa.
\label{eq:profile-root-error}
\end{equation}
A width-based termination has error at most
\(C_{\rm tune}\eta/2\), since its bracket still contains the root.
Consequently every cell--anchor output satisfies
\begin{equation}
  |\widetilde y_z(a)-\overline y_z^*(a)|
  \le\max\{C_{\rm tune}/2,3/\kappa\}\eta.
\label{eq:profile-center-error}
\end{equation}
We now transfer the anchor estimate to a decision price
\(p\in\Gamma_{\rm dec}\).  Since \(\phi\) has bounded derivative and
\(f_x\ge\kappa\) along the quantile curve \(\{z_q(x):q\in\cP\}\),
\[
  \kappa|z_p(x)-z_a(x)|
  \le |F_x(z_p(x))-F_x(z_a(x))|
  =|\phi(p)-\phi(a)|
  \le C|p-a|.
\]
Here the first inequality follows by integrating the density between the two
quantiles. Every intermediate quantile is \(z_q(x)\) for a price between
\(p\) and \(a\).  Thus \(p\mapsto z_p(x)\) is uniformly Lipschitz under the
context-conditioned shared-noise assumption.  With \(a=\pi(p)\),
\(|p-a|\le\sqrt\varepsilon\).  Since
\cref{lem:unique-inventory} gives
\(y_x^*(p)=\lambda(x,p)+z_p(x)\), for
\(x\in C_z^{\rm prof}\),
\begin{equation}
\begin{aligned}
  |\widetilde y_z(a)-y_x^*(p)|
  &\le |\widetilde y_z(a)-\overline y_z^*(a)|
       +|\overline y_z^*(a)-y_x^*(a)|
       +|\lambda(x,a)-\lambda(x,p)|+|z_a(x)-z_p(x)|\\
  &\le\max\{C_{\rm tune}/2,3/\kappa\}\eta
      +L_{\rm loc}\sqrt\varepsilon+(L_p+L_\phi/\kappa)|p-a|\\
  &\le C_{\rm prof}\sqrt\varepsilon.
\end{aligned}
\label{eq:profile-query-error}
\end{equation}

The error constant is therefore
\[
 C_{\rm prof}=1+L_{\rm loc}+L_p+L_\phi/\kappa
       +c_\eta\max\{C_{\rm tune}/2,3/\kappa\}.
\]
Only the last term is reduced by decreasing $c_\eta$: the cell and anchor
approximation errors retain resolution $\sqrt\varepsilon$.

It remains to count the rounds that actually perform profiling.  Writing
$M=|\Gamma_y||\mathcal Z_{\rm prof}|$, the exact deterministic cap is
\[
 N_{\rm prof}(T)\le\min\{T,MB_y m_y\},\qquad
 m_y=\left\lceil C_{\rm tune}c_\eta^{-2}\varepsilon^{-1}
             \log(TMB_y/\delta)\right\rceil,
\]
where $B_y$ uses $C_{\rm tune}c_\eta\sqrt\varepsilon$ in its denominator.
Thus the batch cost scales as $c_\eta^{-2}$ before rounding, with an
additional logarithmic dependence through $B_y$.  For each fixed positive
$c_\eta$, these factors can be absorbed into constants in the following
bounds, preserving the exponent of $\varepsilon$.  Each
cell--anchor pair uses at most \(B_y\) batches, and each batch uses \(m_y\)
arrivals assigned to that cell.  Hence every cell uses at most
\begin{equation}
  R_{\rm prof}
  \le
  C|\Gamma_y|\varepsilon^{-1}
  \log(e/\varepsilon)
  \log\!\left(
    \frac{T|\Gamma_y||\mathcal Z_{\rm prof}|}{\delta}
  \right)
\label{eq:profile-cell-budget}
\end{equation}
profiling actions.  Since
\(|\mathcal Z_{\rm prof}|\le C\delta_x^{-d}\), summing this
deterministic cap over cells gives, for every arrival realization and every
horizon,
\begin{equation}
  N_{\rm prof}(T)
  \le
  C|\Gamma_y|\delta_x^{-d}\varepsilon^{-1}
  \log^2\!\left(
    \frac{CT|\Gamma_y|\delta_x^{-d}}{\delta}
  \right).
\label{eq:profile-action-budget}
\end{equation}
An unfinished cell cannot create waiting loss elsewhere: each arrival to an
unfinished cell is one of the actions counted above, and each arrival to a
finished cell is a UCB round.  On \(\mathcal E_{\rm bis}\), every cell that
does finish has the accuracy guarantees in
\cref{eq:profile-center-error,eq:profile-query-error}.  Cells that have not
received enough arrivals remain in profiling mode and never supply an
uncertified inventory to the UCB history.

Under \(h\le h_{\rm vis}\), the certified bound on
\cref{eq:profile-query-error} lies inside the radius in
\cref{ass:inventory}.  The quadratic loss bound in
\cref{lem:inventory-consequences} and \(\varepsilon=h^\alpha\le1\) give
\[
  G_x(p)-Q_x(p,\widetilde y_x(p))
  \le C_y C_{\rm prof}^2\varepsilon
  =B_{\rm prof}\varepsilon.
\]
The equality of the proxy and true inventory gaps follows from
\cref{prop:proxy-identity}.  Finally,
\(y_x^*(p)-q_\rho(x,p)\ge m_q\), and the definition of
\(h_{\rm vis}\) makes the right-hand side of
\cref{eq:profile-query-error} at most \(m_q/2\).  Hence
\(\widetilde y_x(p)\ge q_\rho(x,p)+m_q/2\).  At every
\(u\le q_\rho(x,p)\), censoring is therefore inactive, so the
\(\rho\)-quantile of sales equals the \(\rho\)-quantile of demand.
\end{proof}

\begin{proof}[Proof of \cref{prop:inventory-learning}]
Fix \(p\) and apply the bisection procedure with the singleton anchor set
\(\Gamma_y=\{p\}\).  For sufficiently small \(\varepsilon\),
\(L_{\rm loc}\delta_x^\alpha\le L_{\rm loc}\sqrt\varepsilon\le r_q/4\).
By \cref{lem:context-grid}, each positive-probability cell has a unique
target \(\overline y_C(p)\), characterized by
\(\overline S_C(\overline y_C(p)\mid p)=\tau(p)\), and
\[
  \sup_{x\in C}|\overline y_C(p)-y_x^*(p)|
  \le L_{\rm loc}\delta_x^\alpha.
\]
The concentration and bisection-invariant argument in the proof of
\cref{lem:profile} applies unchanged to any partition with
\(\delta_x^\alpha\le\sqrt\varepsilon\).  In particular,
\cref{eq:profile-center-error}, with \(\eta=c_\eta\sqrt\varepsilon\),
gives, simultaneously for every completed cell, with probability at
least \(1-\delta\),
\[
  |\widehat y_C(p)-\overline y_C(p)|
  \le C_0\sqrt\varepsilon,
  \qquad
  C_0:=c_\eta\max\{C_{\rm tune}/2,3/\kappa\}.
\]
Combining this bound with \cref{lem:context-grid} by the triangle
inequality gives
\[
  \sup_{x\in C}|\widehat y_C(p)-y_x^*(p)|
  \le (C_0+L_{\rm loc})\sqrt{\varepsilon}
  =C_{\rm inv}\sqrt{\varepsilon},
  \qquad
  C_{\rm inv}:=C_0+L_{\rm loc}.
\]
This constant depends only on the model constants and fixed algorithm
parameters, not on \(T\), \(\varepsilon\), or \(\delta_x\).

Each cell uses at most \(B_y m_y=\widetilde O(\varepsilon^{-1})\)
inventory-search rounds, as in \cref{eq:profile-cell-budget} with one anchor.
Each such round incurs regret at most \(R_{\max}\), so the cumulative
regret incurred during inventory search in each completed cell is at most
\[
  R_{\max}B_y m_y=\widetilde O(\varepsilon^{-1}).
\]
The same upper bound applies to the search rounds already used in an
unfinished cell. Zero-probability cells contribute no rounds almost surely.
\end{proof}

\section{Proofs for Kernel Concentration and Counting}
\label{app:kernel-concentration}

\begin{lemma}[Sequential kernel proxy and quantile concentration]
\label{lem:kernel-quantile}
Let \(\ell_T(\delta)=\log(CT^4|\Gamma|h^{-d}/\delta)\).  There is an event
\(\mathcal E_{\rm ker}\) with probability at least \(1-\delta\) such that,
on its intersection with the profiling event, uniformly over UCB queries with
\(W_t(x,p)>0\),
\[
  |\widehat{\ProxyG}_t(x,p)-\ProxyG_x(p)|
  \le C\left(\sqrt{\frac{\ell_T(\delta)}{W_t(x,p)}}+\varepsilon+h^\alpha\right),
\]
\[
  |\widehat q_{\rho,t}(x,p)-q_\rho(x,p)|
  \le C\left(\sqrt{\frac{\ell_T(\delta)}{W_t(x,p)}}+h^\alpha\right).
\]
\end{lemma}

\begin{lemma}[Kernel calibrated-value confidence]
\label{lem:confidence}
On the intersection of the profiling event and
\(\mathcal E_{\rm ker}\), uniformly over UCB queries with \(W_t(x,p)>0\),
\[
  \left|\widehat G_t^{\rm cal}(x,p)-\{G_x(p)-bz_\rho(x)\}\right|
  \le C\sqrt{\frac{\ell_T(\delta)}{\max\{1,W_t(x,p)\}}}
  +C(\varepsilon+h^\alpha).
\]
More precisely, write \(H_x(p)=G_x(p)-bz_\rho(x)\) and let
\[
  \mathfrak b_t(x,p):=
  \frac{\sum_{s\in\mathcal T_t(p)}K((x_s-x)/h)
    \{G_{x_s}(p)-Q_{x_s}(p,y_s)\}}{W_t(x,p)}.
\]
Then \(0\le\mathfrak b_t(x,p)\le B_{\rm prof}\varepsilon\), and
\begin{equation}
\begin{aligned}
  \widehat G_t^{\rm cal}(x,p)-H_x(p)
  &=-\mathfrak b_t(x,p)+e_t^{\rm cal}(x,p),\\
  |e_t^{\rm cal}(x,p)|
  &\le C_{\rm ker}\left\{
    \sqrt{\frac{\ell_T(\delta)}{\max\{1,W_t(x,p)\}}}+h^\alpha
  \right\}.
\end{aligned}
\label{eq:calibrated-signed-error}
\end{equation}
Here \(C_{\rm ker}\) depends only on the fixed model envelopes and not on
\(C_{\rm tune}\) or \(B_{\rm prof}\).  The first, symmetric bound may
also depend on the fixed \(B_{\rm prof}\).
\end{lemma}

\begin{lemma}[Kernel counting]
\label{lem:counting}
Let \(\mathcal U_T\) be the set of UCB rounds.  For every predictable
interleaving rule and price sequence \(p_t\in\Gamma\) on those rounds,
\[
  \sum_{t\in\mathcal U_T}
  \frac{1}{\sqrt{\max\{1,W_t(x_t,p_t)\}}}
  \le C\sqrt{T|\Gamma|h^{-d}}.
\]
\end{lemma}

\begin{proof}[Proof of \cref{lem:kernel-quantile}]
We first make the sequential measurability explicit for the interleaved
policy.  Include the policy's internal random seed and initialized profiler
state in \(\mathcal F_0\), and let
\[
  \mathcal F_{s-1}
  =\sigma\!\left(
    \mathcal F_0,(x_r,p_r,y_r,O_r):r<s
  \right),
  \qquad
  \mathcal G_s=\mathcal F_{s-1}\vee\sigma(x_s).
\]
The decision whether round \(s\) is a profiling or UCB round, and then the
action \((p_s,y_s)\), are \(\mathcal G_s\)-measurable.  By
\cref{ass:context}, \(\epsilon_s\mid\mathcal G_s\sim F_{x_s}\).  Write
\(A_s=1\) when round \(s\) is a UCB round.

Fix a query time \(t\), a price \(p\in\Gamma\), and condition on
\(x_t=x\).  For \(s<t\), enlarge the pre-outcome sigma-field to
\[
  \mathcal G_s^{(t)}=\mathcal G_s\vee\sigma(x_t).
\]
By the i.i.d.\ pair assumption and independence from policy randomization
in \cref{ass:context}, \(x_t\) is independent of the policy's random seed
and all pairs \((x_r,\epsilon_r)\) with \(r<t\).  Thus
\(\epsilon_s\mid\mathcal G_s^{(t)}\sim F_{x_s}\), so adjoining
\(x_t\) preserves the conditional law of the period-\(s\) outcome.  Set
\[
  w_s=K((x_s-x)/h)\one\{A_s=1,p_s=p\},
  \qquad
  m_s^Y=\E[Y_s\mid\mathcal G_s].
\]
The weight is \(\mathcal G_s^{(t)}\)-measurable, and
\(\E[Y_s-m_s^Y\mid\mathcal G_s^{(t)}]=0\).  Hence
\(w_s(Y_s-m_s^Y)\) is a bounded martingale difference in the filtration that
reveals \(O_s\) after \(\mathcal G_s^{(t)}\).  Since
\(|Y_s-m_s^Y|\le2Y_{\max}\) and \(0\le w_s\le1\), Hoeffding's conditional
lemma gives, for every \(\theta\in\R\),
\begin{equation}
  \E\!\left[
    \exp\!\left\{
      \theta w_s(Y_s-m_s^Y)
      -2\theta^2Y_{\max}^2w_s^2
    \right\}
    \middle|\mathcal G_s^{(t)}
  \right]
  \le1.
\label{eq:kernel-exponential-step}
\end{equation}
Thus the product of the factors in
\cref{eq:kernel-exponential-step} is a nonnegative supermartingale.  Also,
\[
  \sum_{s<t}w_s^2\le \sum_{s<t}w_s=W_t(x,p).
\]
For \(W_t(x,p)>1\), use the dyadic events
\(2^{j-1}<W_t(x,p)\le2^j\),
\(j=1,\ldots,\lceil\log_2T\rceil\).  On a given level, the exponential
bound with its deterministic variance cap \(2^j\) gives
\[
  \Pbb\!\left(
    \left|\sum_{s<t}w_s(Y_s-m_s^Y)\right|
       >C\sqrt{2^j L},\ W_t(x,p)\le2^j
    \,\middle|\,x_t=x\right)\le2e^{-L}.
\]
Indeed, apply Markov's inequality to the exponential supermartingale and
choose \(|\theta|\) proportional to \(\sqrt{L/2^j}/Y_{\max}\).
Within this level \(2^j<2W_t(x,p)\).  Assigning failure probability
\(\delta/(2CT|\Gamma|\log(eT))\) to each query and level therefore gives
\begin{equation}
  \left|
    \sum_{s<t}w_s(Y_s-m_s^Y)
  \right|
  \le C\sqrt{W_t(x,p)\ell_T(\delta)}
\label{eq:proxy-martingale}
\end{equation}
for this conditional query.  For \(0<W_t(x,p)\le1\), boundedness gives
the pathwise inequality
\[
  \left|\sum_{s<t}w_s(Y_s-m_s^Y)\right|
  \le2Y_{\max}W_t(x,p)
  \le2Y_{\max}\sqrt{W_t(x,p)\ell_T(\delta)}.
\]
Thus \cref{eq:proxy-martingale} covers all positive weights.  After
division by \(W_t(x,p)>0\), the stochastic proxy error is
at most \(C\sqrt{\ell_T(\delta)/W_t(x,p)}\).  This derivation handles the
data-dependent effective sample size over its full positive range.

On the profiling event, every observation with \(A_s=1\) comes from a cell
whose profile was already completed before choosing the action.  Hence
\cref{prop:proxy-identity,lem:profile} give
\begin{equation}
  |m_s^Y-\ProxyG_{x_s}(p)|
  =G_{x_s}(p)-Q_{x_s}(p,\widetilde y_{x_s}(p))
  \le C\varepsilon.
\label{eq:proxy-profile-bias}
\end{equation}
Moreover, the context-regularity conclusion of
\cref{prop:location-separation} implies, whenever \(w_s>0\),
\begin{equation}
  |\ProxyG_{x_s}(p)-\ProxyG_x(p)|
  \le L_{\rm loc}h^\alpha.
\label{eq:proxy-context-bias}
\end{equation}
The weighted average of
\cref{eq:proxy-profile-bias,eq:proxy-context-bias}, combined with
\cref{eq:proxy-martingale}, proves
\begin{equation}
  |\widehat{\ProxyG}_t(x,p)-\ProxyG_x(p)|
  \le
  C\left\{
    \sqrt{\frac{\ell_T(\delta)}{W_t(x,p)}}
    +\varepsilon+h^\alpha
  \right\}.
\label{eq:proxy-complete-bound}
\end{equation}

We now control the weighted empirical quantile.  Put
\[
  r_\rho:=\min\{r_q,m_q/4,r_y/2\},
  \mathcal U_{x,p}
  =[q_\rho(x,p)-r_\rho,q_\rho(x,p)+r_\rho].
\]
When \(w_s>0\), context smoothness gives
\[
  |q_\rho(x_s,p)-q_\rho(x,p)|
  \le (L_x+L_F/\kappa)h^\alpha.
\]
Together with the profiling margin
\(\widetilde y_{x_s}(p)\ge q_\rho(x_s,p)+m_q/2\) and the certified inequality
\((L_x+L_F/\kappa)h^\alpha\le m_q/8\), this shows that
\[
  u\le q_\rho(x,p)+r_\rho
  \le q_\rho(x_s,p)+r_\rho+m_q/8
  <\widetilde y_{x_s}(p)
\]
throughout \(\mathcal U_{x,p}\).  Thus censoring does not alter the lower-tail
indicators:
\begin{equation}
  \one\{O_s\le u\}=\one\{D_s\le u\},
  \qquad u\in\mathcal U_{x,p}.
\label{eq:cdf-uncensored-neighborhood}
\end{equation}

Let \(\mathcal V_T\) be a deterministic grid of \([0,y_{\max}]\) with mesh
at most \(T^{-2}\), so \(|\mathcal V_T|\le CT^2\).  Conditional on
\(x_t=x\), augment the grid points lying in \(\mathcal U_{x,p}\) by the two
endpoints of that interval.  This query-specific grid still has at most
\(CT^2+2\) points, all inside the uncensored neighborhood.  Fix one such
point \(v\).  Construct the concentration event using latent demand,
before intersecting it with the profiling event.  Specifically, enlarge the
analysis filtration to
\[
  \mathcal F_{s-1}^{D}
  :=\sigma\!\left(\mathcal F_0,(x_r,p_r,y_r,D_r):r<s\right),
  \qquad
  \mathcal J_s^{(t)}:=\mathcal F_{s-1}^{D}\vee\sigma(x_s,x_t).
\]
Observed sales and all policy decisions are measurable in this larger
analysis filtration.  The policy continues to use its observed-sales history.
For \(s<t\), the i.i.d.\ pairs and independent policy randomization in
\cref{ass:context} give
\(\epsilon_s\mid\mathcal J_s^{(t)}\sim F_{x_s}\), including after
conditioning on latent past demand and the query context.  Define
\[
  \xi_s^D(v)
  :=\one\{D_s\le v\}-F_{x_s}(v-\lambda(x_s,p_s)).
\]
The weights remain predictable and
\(\E[w_s\xi_s^D(v)\mid\mathcal J_s^{(t)}]=0\) on the full probability
space.  Consequently \(w_s\xi_s^D(v)\) obeys the same
exponential-supermartingale argument as the proxy sequence, without any
conditioning on profiling success.  For \(0<W_t\le1\), use the
deterministic bound \(|\sum_{s<t}w_s\xi_s^D(v)|\le W_t\).
Allocate failure probability
\(\delta/(2CT|\Gamma||\mathcal V_T|\log(eT))\) to every
\((t,p,v)\) and dyadic weight level.  A union bound gives the desired
inequality for the latent weighted CDF at all relevant grid points.
We then intersect this concentration event with the profiling event.
On the intersection,
\cref{eq:cdf-uncensored-neighborhood} identifies the latent and observed
indicators for every positive weight, and \(p_s=p\) on those weights.
Thus the unconditional concentration bound for latent demand transfers to
the empirical sales CDF on the profiling event.

For an arbitrary \(u\in\mathcal U_{x,p}\), take adjacent grid points
\(v_-\le u\le v_+\).  Monotonicity gives
\(\widehat F_t(v_-\mid x,p)\le\widehat F_t(u\mid x,p)
\le\widehat F_t(v_+\mid x,p)\), while the population CDF changes by at most
\(KT^{-2}\) between adjacent points.  Absorbing this deterministic term into
the radius yields, simultaneously over \(u\in\mathcal U_{x,p}\),
\begin{equation}
  \left|
    \widehat F_t(u\mid x,p)
    -\frac{\sum_{s<t}w_sF_{x_s}(u-\lambda(x_s,p))}{W_t(x,p)}
  \right|
  \le C\sqrt{\frac{\ell_T(\delta)}{W_t(x,p)}}.
\label{eq:weighted-cdf-random}
\end{equation}
The upper density bound, the H\"older condition on \(F_x\), and
\(|\lambda(x_s,p)-\lambda(x,p)|\le L_xh^\alpha\) imply
\[
\begin{aligned}
  &\left|
    F_{x_s}(u-\lambda(x_s,p))-F_x(u-\lambda(x,p))
  \right|\\
  &\qquad\le
  \sup_v|F_{x_s}(v)-F_x(v)|
  +K|\lambda(x_s,p)-\lambda(x,p)|\\
  &\qquad\le (L_F+KL_x)h^\alpha.
\end{aligned}
\]
Consequently,
\begin{equation}
  \sup_{u\in\mathcal U_{x,p}}
  |\widehat F_t(u\mid x,p)-F_x(u-\lambda(x,p))|
  \le e_t,
  \qquad
  e_t=C\left\{
    \sqrt{\frac{\ell_T(\delta)}{W_t(x,p)}}+h^\alpha
  \right\}.
\label{eq:weighted-cdf-total}
\end{equation}

The target CDF equals \(\rho\) at \(q_\rho(x,p)\).  On the fixed
neighborhood from \cref{ass:inventory}, its density is at least
\(\kappa\).  Suppose first that
\begin{equation}
  \frac{2e_t}{\kappa}
  \le r_\rho.
\label{eq:quantile-small-error-regime}
\end{equation}
Integrating the density lower bound on both sides of \(z_\rho(x)\) gives
\[
  F_x(q_\rho(x,p)+2e_t/\kappa-\lambda(x,p))
  \ge\rho+2e_t
\]
and
\[
  F_x(q_\rho(x,p)-2e_t/\kappa-\lambda(x,p))
  \le\rho-2e_t.
\]
Both evaluation points lie in \(\mathcal U_{x,p}\), so
\cref{eq:weighted-cdf-total} implies that the empirical CDF is below \(\rho\)
at the lower point and above \(\rho\) at the upper point.  The defining
crossing property of the measurable weighted empirical quantile then yields
\begin{equation}
  |\widehat q_{\rho,t}(x,p)-q_\rho(x,p)|
  \le 2e_t/\kappa.
\label{eq:quantile-inversion}
\end{equation}
Outside the regime in \cref{eq:quantile-small-error-regime},
\[
  e_t\ge e_0,
  \qquad
  e_0:=\frac\kappa2r_\rho>0.
\]
Both quantiles lie in the fixed bounded sales interval, so enlarging the
universal radius constant by that interval's diameter divided by \(e_0\)
gives the same order \(C e_t\) there as well.

The conditional failure bounds hold uniformly in \(x_t=x\).
Integrating over the query context gives the corresponding unconditional
bounds.  A union bound
over at most \(T|\Gamma|\) realized time--price queries, at most
\(1+\lceil\log_2T\rceil\) dyadic weight levels, and
\(|\mathcal V_T|\le CT^2\) CDF grid points has total failure probability at
most \(\delta\).  The logarithm of this cardinality is bounded by
\(\ell_T(\delta)\), which completes both simultaneous bounds over the
realized queries.
\end{proof}

\begin{proof}[Proof of \cref{lem:confidence}]
Fix a query with positive weight and write
\(w_s=K((x_s-x)/h)\one\{s\in\mathcal T_t(p)\}\).
The proxy identity gives
\(m_s^Y=\ProxyG_{x_s}(p)-\{G_{x_s}(p)-Q_{x_s}(p,y_s)\}\)
whenever \(w_s>0\).  Hence the following decomposition is exact:
\[
\begin{aligned}
  \widehat G_t^{\rm cal}(x,p)-H_x(p)
  ={}&-\mathfrak b_t(x,p)
    +\frac{\sum_{s<t}w_s(Y_s-m_s^Y)}{W_t(x,p)}\\
    &+\frac{\sum_{s<t}w_s\{\ProxyG_{x_s}(p)-\ProxyG_x(p)\}}
             {W_t(x,p)}
    -b\{\widehat q_{\rho,t}(x,p)-q_\rho(x,p)\}.
\end{aligned}
\]
The last three terms define \(e_t^{\rm cal}\).  For \(W_t\ge1\),
\cref{eq:proxy-martingale,eq:proxy-context-bias,lem:kernel-quantile}
bound them by the right-hand side of \cref{eq:calibrated-signed-error},
with a constant independent of the profiler's stopping tolerance.
For \(0<W_t<1\), the first of these terms is at most \(2Y_{\max}\)
in magnitude, the second is at most \(L_{\rm loc}h^\alpha\), and the
quantile difference is at most the fixed sales-range diameter.  Since
\(\ell_T(\delta)\ge1\), these deterministic bounds give the same result
with denominator \(\max\{1,W_t\}\).  Finally, profiling bounds each
inventory gap by \(B_{\rm prof}\varepsilon\). Their weighted average
lies in \([0,B_{\rm prof}\varepsilon]\).  This proves the signed
decomposition and, by the triangle inequality, the symmetric confidence bound.
\end{proof}

\begin{proof}[Proof of \cref{lem:counting}]
Use a deterministic partition of \([0,1]^d\) into
\(M\le Ch^{-d}\) cells of diameter at most \(h/2\), and let \(j(t)\) denote
the cell containing \(x_t\) on a UCB round.  Two contexts in the same cell
are at distance at most \(h/2\).  Hence every earlier UCB observation in cell
\(j(t)\) at price \(p_t\) has kernel weight at least \(k_0\).  Let
\(N_{j,p}(t)\) be the number of such UCB observations strictly before \(t\).
Pathwise,
\[
  W_t(x_t,p_t)\ge k_0N_{j,p_t}(t).
\]
Order the visits to a fixed pair \((j,p)\) by
\(r=1,\ldots,N_{j,p}(T)\).  On its \(r\)-th visit, the preceding count is
\(r-1\).  The first visit contributes one because the denominator in the
lemma is truncated at one, and
\[
  \sum_{r=1}^{N_{j,p}(T)}
  \frac1{\sqrt{\max\{1,r-1\}}}
  \le C\sqrt{N_{j,p}(T)}.
\]
After absorbing \(k_0^{-1/2}\), summing this bound and applying
Cauchy--Schwarz gives
\[
\begin{aligned}
  \sum_{t\in\mathcal U_T}
  \frac1{\sqrt{\max\{1,W_t(x_t,p_t)\}}}
  &\le C\sum_{j,p}\sqrt{N_{j,p}(T)}\\
  &\le C\sqrt{M|\Gamma|\sum_{j,p}N_{j,p}(T)}\\
  &\le C\sqrt{|\mathcal U_T||\Gamma|h^{-d}}
   \le C\sqrt{T|\Gamma|h^{-d}}.
\end{aligned}
\]
This argument is deterministic once the context and price sequences are
fixed, so it applies to every predictable adaptive interleaving policy.
\end{proof}

\section{Complete Proofs of the Contextual Upper Bounds}
\label{app:upper-complete}

\begin{proof}[Complete proof of \cref{thm:upper}]
Let
\(r=(d+2\alpha)/(d+3\alpha)<1\) and
\(K_{\rm dec}=|\Gamma_{\rm dec}|\),
\(K_{\rm prof}=|\Gamma_y|\).  The tuning constants are fixed
problem constants.  Since \(h_T\to0\) and
\[
  K_{\rm prof}\delta_x^{-d}\varepsilon^{-1}\polylog(T)=o(T),
\]
there is a finite \(T_0\) such that
\(h_T\le h_{\rm vis}\) for every \(T\ge T_0\).
For \(T<T_0\), every branch of \cref{alg:mck} has regret at most
\(R_{\max}T\le R_{\max}T_0^{1-r}T^r\), so increasing the theorem constant
covers these horizons.  We henceforth take \(T\ge T_0\). The deterministic
startup fallback is then inactive.

Run the interleaved profiler--UCB policy with \(\delta=T^{-2}\), and let
\(\mathcal E_{\rm prof}\) and
\(\mathcal E_{\rm ker}\) denote the events from
\cref{lem:profile,lem:kernel-quantile}.  For fixed \(d,\alpha\), the
logarithm in the index is equivalent up to a constant to
\(\ell_T(T^{-2})\).  The choice in \cref{app:tuning} makes the index
radius dominate the \(C_{\rm ker}\) bound on \(e_t^{\rm cal}\).
The possibly larger inventory bias \(B_{\rm prof}\varepsilon\) is
kept separately below.  A union bound gives
\begin{equation}
  \Pbb(\mathcal E^c)\le2T^{-2},
  \qquad
  \mathcal E:=\mathcal E_{\rm prof}\cap\mathcal E_{\rm ker}.
\label{eq:upper-good-event}
\end{equation}
On \(\mathcal E^c\), total regret is at most \(TR_{\max}\), so its expected
contribution is at most \(2R_{\max}/T\).

For every arrival and outcome realization, \cref{lem:profile} bounds the
number of profiling rounds by
\[
  N_{\rm prof}(T)
  \le
  CK_{\rm prof}\delta_x^{-d}\varepsilon^{-1}\polylog(T).
\]
Every one of these rounds costs at most \(R_{\max}\), so
\begin{equation}
  \Reg_{\rm prof}
  \le
  CK_{\rm prof}\delta_x^{-d}\varepsilon^{-1}\polylog(T).
\label{eq:upper-profile-regret}
\end{equation}
Here and below the constants may depend on the fixed \(c_\eta\),
including the batch factor \(c_\eta^{-2}\) and
\(B_{\rm prof}=C_y C_{\rm prof}^2\).
For the stated bandwidth this quantity is \(o(T)\).  Arrivals to completed
cells are UCB rounds and enter the UCB regret bound below.

Consider a UCB round with \(x_t=x\).  Let \(p_x^\Gamma\) maximize
\(G_x\) over the grid.  By \cref{prop:uniform-price-grid},
\begin{equation}
  V_x-G_x(p_x^\Gamma)\le L_G\varepsilon.
\label{eq:upper-grid-gap}
\end{equation}
Recall \(H_x(p)=G_x(p)-bz_\rho(x)\), and denote the actual finite
index radius in \cref{alg:mck} by
\[
  \beta_t(x,p)
  :=C_{\rm tune}\sqrt{
    \frac{\log(T|\Gamma_{\rm dec}|/h)}{\max\{1,W_t(x,p)\}}}
    +C_{\rm tune}(\varepsilon+h^\alpha).
\]
On \(\mathcal E\), \cref{eq:calibrated-signed-error} and the tuning
choice give \(|e_t^{\rm cal}(x,p)|\le\beta_t(x,p)\).  Therefore, for
every positive-weight price,
\begin{equation}
  H_x(p)-B_{\rm prof}\varepsilon
  \le U_t(x,p)\le H_x(p)+2\beta_t(x,p).
\label{eq:upper-approximate-optimism}
\end{equation}
The lower bound uses \(\mathfrak b_t\le B_{\rm prof}\varepsilon\).
The upper bound uses \(\mathfrak b_t\ge0\).
For analysis only, put
\(U_t^+(x,p):=U_t(x,p)+B_{\rm prof}\varepsilon\), with
\(+\infty+B_{\rm prof}\varepsilon=+\infty\).  This added quantity
is independent of the candidate price, so the policy's chosen maximizer is
also a maximizer of \(U_t^+\).  The policy continues to use the original
indices \(U_t\).  Whenever both relevant effective sample
sizes are positive,
\begin{equation}
\begin{aligned}
  H_x(p_x^\Gamma)
  &\le U_t^+(x,p_x^\Gamma)
   \le U_t^+(x,p_t)\\
  &\le H_x(p_t)+2\beta_t(x,p_t)+B_{\rm prof}\varepsilon.
\end{aligned}
\label{eq:upper-optimism-chain}
\end{equation}
The first inequality follows from the lower confidence bound after the
common shift, the second from the UCB selection rule, and the last from
the upper confidence bound after the same shift.  When
\(W_t(x,p_t)>0\), a candidate with zero effective sample size has infinite
index, so \(p_x^\Gamma\) cannot have zero weight. Hence the displayed chain
indeed covers this case.  When \(W_t(x,p_t)=0\), boundedness gives
\(G_x(p_x^\Gamma)-G_x(p_t)\le C\), which is absorbed by the same right-hand
side because \(\max\{1,W_t(x,p_t)\}=1\) and
\(\ell_T(T^{-2})\ge1\).  Since both \(B_{\rm prof}\) and
\(C_{\rm tune}\) are fixed constants and the within-context shift cancels,
\cref{eq:upper-optimism-chain} therefore yields for every UCB round
\begin{equation}
  G_x(p_x^\Gamma)-G_x(p_t)
  \le
  C\sqrt{
    \frac{\ell_T(T^{-2})}
         {\max\{1,W_t(x,p_t)\}}
  }
  +C(\varepsilon+h^\alpha).
\label{eq:upper-price-gap}
\end{equation}

The period regret decomposes exactly as
\begin{equation}
\begin{aligned}
  V_x-Q_x(p_t,\widetilde y_x(p_t))
  ={}&\{V_x-G_x(p_x^\Gamma)\}
     +\{G_x(p_x^\Gamma)-G_x(p_t)\}\\
    &+\{G_x(p_t)-Q_x(p_t,\widetilde y_x(p_t))\}.
\end{aligned}
\label{eq:upper-period-decomposition}
\end{equation}
The three terms account for price discretization, price selection, and
inventory approximation, respectively.  The first two are bounded by
\cref{eq:upper-grid-gap,eq:upper-price-gap}, and \cref{lem:profile} bounds
the third by \(C\varepsilon\).  Combining these bounds gives
\begin{equation}
  V_x-Q_x(p_t,\widetilde y_x(p_t))
  \le
  C\sqrt{
    \frac{\ell_T(T^{-2})}
         {\max\{1,W_t(x,p_t)\}}
  }
  +C(\varepsilon+h^\alpha).
\label{eq:upper-one-round}
\end{equation}
We sum \cref{eq:upper-one-round} over the UCB rounds and apply
\cref{lem:counting} to the confidence-width term.  Adding the profiling
regret in \cref{eq:upper-profile-regret} and the failure-event contribution
gives
\begin{equation}
  \Reg_T
  \le
  C\polylog(T)
  \left\{
    K_{\rm prof}\delta_x^{-d}\varepsilon^{-1}
    +\sqrt{TK_{\rm dec}h^{-d}}
    +T(\varepsilon+h^\alpha)
  \right\}.
\label{eq:upper-master}
\end{equation}

For the Lipschitz price class,
\(K_{\rm dec}\le C\varepsilon^{-1}\),
\(K_{\rm prof}\le C\varepsilon^{-1/2}\), and
\(\delta_x^{-d}=h^{-d/2}\).  With
\(\varepsilon=h^\alpha\), \cref{eq:upper-master} becomes
\[
  \Reg_T
  \le
  C\polylog(T)
  \left\{
    h^{-(d+3\alpha)/2}
    +\sqrt{Th^{-d-\alpha}}
    +Th^\alpha
  \right\}.
\]
Ignoring logarithmic factors, the choice
\(h=T^{-1/(d+3\alpha)}\) is equivalent to
\(T=h^{-(d+3\alpha)}\), and consequently
\[
  h^{-(d+3\alpha)/2}
  =\sqrt T,
  \qquad
  \sqrt{Th^{-d-\alpha}}
  =Th^\alpha
  =T^{(d+2\alpha)/(d+3\alpha)}.
\]
Since \((d+2\alpha)/(d+3\alpha)>1/2\), the profiling term is lower order.
The logarithmic modification in \cref{alg:mck} changes only the polylogarithmic
factor, completing the proof.
\end{proof}

\begin{proof}[Complete proof of \cref{cor:smooth-upper}]
Let \(r_{\rm sm}=(2d+3\alpha)/(2d+5\alpha)<1\).  With the smooth bandwidth
and grid, \(h_T\to0\) and the profiling-action bound below is \(o(T)\), so
there is a finite \(T_{0,{\rm sm}}\) after which the deterministic visibility
fallback is inactive.  Earlier horizons satisfy
\(R_{\max}T\le R_{\max}T_{0,{\rm sm}}^{1-r_{\rm sm}}T^{r_{\rm sm}}\).
For later horizons, we apply the event bounds, profiling budget, optimism
inequalities, and counting argument in
\cref{eq:upper-good-event,eq:upper-master} and the intervening steps.
Under the bounded-second-derivative condition, \cref{prop:uniform-price-grid} permits a
mesh of width \(\sqrt\varepsilon\).  It has
\(K_{\rm dec}=O(\varepsilon^{-1/2})\) and price approximation loss
\(O(\varepsilon)\).  Therefore \cref{eq:upper-master} becomes
\[
  \Reg_T
  \le
  C\polylog(T)
  \left\{
    h^{-d/2}\varepsilon^{-3/2}
    +\sqrt{Th^{-d}\varepsilon^{-1/2}}
    +T(\varepsilon+h^\alpha)
  \right\}.
\]
Substitute \(\varepsilon=h^\alpha\).  Ignoring logarithms, the bandwidth
\(h=T^{-2/(2d+5\alpha)}\) satisfies
\(T=h^{-(d+5\alpha/2)}\), so
\[
\begin{aligned}
  h^{-d/2}\varepsilon^{-3/2}
  &=h^{-(d+3\alpha)/2},\\
  \sqrt{Th^{-d}\varepsilon^{-1/2}}
  &=h^{-d-3\alpha/2},\\
  T\varepsilon
  &=h^{-d-3\alpha/2}.
\end{aligned}
\]
The profiling exponent \((d+3\alpha)/2\) is no larger than
\(d+3\alpha/2\), and is strictly smaller for \(d>0\).  Finally,
\[
  h^{-d-3\alpha/2}
  =T^{(2d+3\alpha)/(2d+5\alpha)}.
\]
The logarithmic factor in the stated bandwidth is again absorbed by
\(\polylog(T)\).
\end{proof}

\section{Proofs for the Lower Bound Construction}
\label{app:lower-stability}

\begin{lemma}[Slack flat-value baseline]
\label{lem:lower-baseline}
There exist a smooth compactly supported mean-zero density \(f_0\), intervals
\(\cP_0\Subset\operatorname{int}(\cP_1)\Subset\cP\), a smooth
\(\lambda_0\), and \(C_0,s_0>0\) such that the constant family
\(F_x\equiv F_0\) has the residual-law properties of \cref{ass:context}
and the density bounds of \cref{ass:inventory},
\[
  I_H:=4\int_\R\left|\frac{d}{dz}\sqrt{f_0(z)}\right|^2dz<\infty,
\]
\(G_0(p)=p\lambda_0(p)+C_{F_0}(p)\) equals \(C_0\) on \(\cP_0\), is smaller
outside \(\cP_0\), and is at most \(C_0-4s_0\) outside \(\cP_1\).
The baseline can be chosen so that, uniformly in \(p\),
\[
  \lambda_0(p)+\underline z,\quad
  q_{0,\rho}(p)-(y_{\min}+r_y),\quad
  (y_{\max}-r_y)-y_0^*(p),\quad
  z_p^0-z_\rho^0-m_q
  \ge4s_0,
\]
with \(0<r_y\le r_q\), and \(G_0,\lambda_0\) have uniformly bounded first
two price derivatives.
\end{lemma}

\begin{definition}[Normalized context--price packing]
\label{def:lower-packing}
Fix the baseline and the interval \(\cP_0\) from
\cref{lem:lower-baseline}.  Let
\begin{equation}
  \zeta(t)=e^{-1/t}\one\{t>0\},
  \qquad
  \vartheta(u)=
  \frac{\zeta(1/2-|u|)}
       {\zeta(1/2-|u|)+\zeta(|u|-1/4)}.
\label{eq:lower-standard-bump}
\end{equation}
Then \(\vartheta\in C^\infty(\R;[0,1])\), it equals one on
\([-1/4,1/4]\), and its support is \([-1/2,1/2]\).  Define the
normalized context bump
\[
  \Psi(v)=\prod_{\ell=1}^d\vartheta(v_\ell),
  \qquad v\in\R^d.
\]

For \(0<h\le1/4\), put \(m_h=\lfloor(2h)^{-1}\rfloor\), index the
multi-indices \(r\in\{0,\ldots,m_h-1\}^d\) by
\(j=1,\ldots,M_h\), where \(M_h=m_h^d\), and set
\begin{equation}
\begin{aligned}
  c_{j,\ell}&=(2r_\ell+1)h,\qquad \ell=1,\ldots,d,\\
  B_j&=c_j+[-h/2,h/2]^d,\\
  B_j^\circ&=c_j+[-h/4,h/4]^d,\\
  \psi_j(x)&=\Psi\!\left(\frac{x-c_j}{h}\right).
\end{aligned}
\label{eq:lower-context-packing}
\end{equation}
Thus the cells \(B_j\subset[0,1]^d\) are pairwise disjoint,
\(\operatorname{supp}(\psi_j)=B_j\), and \(\psi_j=1\) on
\(B_j^\circ\).

Choose once and for all a closed interval
\(J=[p_L,p_R]\Subset\operatorname{int}(\cP_0)\), and write
\(L_0=p_R-p_L\).  For \(0<w\le L_0/4\), let
\begin{equation}
  K_\Delta=\left\lfloor\frac{L_0}{2w}\right\rfloor,
  \qquad
  p_k=p_L+(2k-1)w,
  \qquad k=1,\ldots,K_\Delta,
\label{eq:lower-price-centers}
\end{equation}
and define
\begin{equation}
  \varphi_0\equiv0,
  \qquad
  \varphi_k(p)=\vartheta\!\left(\frac{p-p_k}{w}\right),
  \qquad
  \mathcal S_k=\{p:\varphi_k(p)>0\}.
\label{eq:lower-price-bumps}
\end{equation}
The sets \(\mathcal S_k\) are pairwise disjoint and contained in
\(J\).  Finally, let
\[
  \Omega_0=\{0,1,\ldots,K_\Delta\}^{M_h},
  \qquad
  \Omega=\{1,\ldots,K_\Delta\}^{M_h},
\]
and, for every label vector \(\omega\in\Omega_0\), define
\begin{equation}
  g_\omega(x,p)
  =\Delta\sum_{j=1}^{M_h}
    \psi_j(x)\varphi_{\omega_j}(p).
\label{eq:lower-gomega-definition}
\end{equation}
The augmented set \(\Omega_0\) contains the zero-label reference
experiments. The final hard prior is supported on \(\Omega\).
\end{definition}

\begin{proposition}[Properties of the normalized packing]
\label{prop:lower-packing-properties}
For the construction in \cref{def:lower-packing}, constants depending only
on \(d,L_0\), and the fixed bump \(\vartheta\) satisfy
\begin{equation}
  c h^{-d}\le M_h\le C h^{-d},
  \qquad
  c w^{-1}\le K_\Delta\le C w^{-1}.
\label{eq:lower-packing-cardinality}
\end{equation}
Under the uniform context distribution,
\begin{equation}
  \Pbb(x_t\in B_j)=h^d,
  \qquad
  \Pbb(x_t\in B_j^\circ)=2^{-d}h^d.
\label{eq:lower-cell-masses}
\end{equation}
Moreover, uniformly over \(\omega\in\Omega_0\),
\begin{equation}
\begin{aligned}
  \norm{g_\omega}_\infty&\le\Delta,\\
  \norm{\partial_p g_\omega}_\infty&\le C\Delta/w,\\
  \norm{\partial_{pp} g_\omega}_\infty&\le C\Delta/w^2,\\
  |g_\omega(x,p)-g_\omega(x',p)|
  &\le C\Delta h^{-\alpha}\norm{x-x'}^\alpha.
\end{aligned}
\label{eq:lower-packing-bounds}
\end{equation}
For \(x\in B_j^\circ\), the perturbation reduces exactly to
\begin{equation}
  g_\omega(x,p)=\Delta\varphi_{\omega_j}(p).
\label{eq:lower-core-identity}
\end{equation}
Consequently, when \(\omega_j=k\ge1\), its maximum over \(\cP_0\) is
\(\Delta\), attained on the plateau of \(\varphi_k\), while it is zero
for every \(p\notin\mathcal S_k\).
\end{proposition}

\begin{lemma}[Stability of lower-bound perturbations]
\label{lem:lower-perturbation-stability}
There are \(h_0,\Delta_0,c_\Delta>0\) and a baseline from
\cref{lem:lower-baseline} such that, for \(h\le h_0\),
\(\Delta=c_\Delta h^\alpha\le\Delta_0\), the packing in
\cref{def:lower-packing} with \(w=\Delta\), and every
\(\omega\in\Omega_0\),
\[
  \lambda_\omega(x,p)=\lambda_0(p)+g_\omega(x,p)/p
\]
is nonnegative, satisfies all assumptions, and has an optimal price in
\(\operatorname{int}(\cP_0)\).  The twice-smooth construction satisfies the
same claim for the packing with \(w=\sqrt{\Delta}\) and uniformly bounded
\(\partial_{pp}G_{\omega,x}\).
\end{lemma}

\begin{proof}[Proof of \cref{lem:lower-baseline}]
Let
\[
  \widetilde f(u)
  =c_f\exp\!\left(-\frac1{1-u^2}\right)\one\{|u|<1\}.
\]
This density is symmetric and hence mean zero, is positive on every compact
subset of \((-1,1)\), and vanishes to all orders at the endpoints.  Therefore
\(\sqrt{\widetilde f}\in H^1(\R)\).  A fixed rescaling
\(f_0(u)=a^{-1}\widetilde f(u/a)\) retains these properties and has finite
\(I_H\).  Since
\(\rho<\inf_{p\in\cP}\phi(p)\) and \(\phi\) is continuous, after taking the
price interval sufficiently small around an interior \(p_0>0\), all relevant
probability levels lie in one compact subset of \((0,1)\).  Their quantiles
therefore lie in a compact subinterval on which \(f_0\) admits uniform
positive lower and finite upper bounds. Choose \(r_q,\kappa\) from this
compact subinterval and a global density upper bound \(K\).
Write \(z_u^0=F_0^{-1}(u)\),
\(z_p^0=z_{\phi(p)}^0\), and let \(C_{F_0}\) denote the residual-value
functional in \cref{prop:location-separation} for \(F_0\). Choose
\(0<m_q<(\inf_p\phi(p)-\rho)/K\), as in
\cref{lem:quantile-separation}.

Choose a location level \(\bar\lambda\), and fixed inventory endpoints, so
that \(\bar\lambda+\underline z>0\) and both
\(\bar\lambda+z_\rho^0\) and \(\bar\lambda+z_p^0\) lie strictly inside the
inventory interval.  This is possible by first taking the noise scale \(a\)
small and then centering the inventory interval at \(\bar\lambda\).  Define
\begin{equation}
  C_0=p_0\bar\lambda+C_{F_0}(p_0),
  \qquad
  \lambda_{\rm flat}(p)=\frac{C_0-C_{F_0}(p)}p.
\label{eq:lower-flat-baseline}
\end{equation}
The function \(C_{F_0}\) is \(C^2\) on the chosen price interval: the quantile
map \(p\mapsto z_p^0\) is \(C^2\) there by the density lower bound, and the
truncated expectations defining \(C_{F_0}\) may be differentiated under the
integral.  Hence \(\lambda_{\rm flat}\) is \(C^2\), equals
\(\bar\lambda\) at \(p_0\), and remains inside all strict feasibility margins
after shrinking the price interval.

Choose
\(\cP_0\Subset\operatorname{int}(\cP_1)\Subset\cP\) around \(p_0\) and a
smooth cutoff \(\chi:\cP\to[0,1]\) that is zero on \(\cP_0\) and one outside
\(\cP_1\), and is strictly positive on \(\cP\setminus\cP_0\).  For a
sufficiently small fixed \(s_0>0\), set
\begin{equation}
  G_0(p)=C_0-4s_0\chi(p),
  \qquad
  \lambda_0(p)=\frac{G_0(p)-C_{F_0}(p)}p.
\label{eq:lower-baseline-extension}
\end{equation}
Continuity on the compact price interval and the strict margins of
\(\lambda_{\rm flat}\) permit choosing \(s_0\) and the interval widths so
that all four inequalities in part (iii) hold.  The cutoff construction gives
part (ii), while \(p\ge p_{\min}>0\) and smoothness of \(G_0,C_{F_0}\) give the
derivative bounds in part (iv).  This proves every item of the lemma.
\end{proof}

\begin{proof}[Proof of \cref{prop:lower-packing-properties}]
For \(h\le1/4\),
\((4h)^{-1}\le m_h\le(2h)^{-1}\), which yields the first cardinality
bound.  Similarly, \(w\le L_0/4\) gives
\(L_0/(4w)\le K_\Delta\le L_0/(2w)\).  The cells and their cores have
Lebesgue volumes \(h^d\) and \((h/2)^d\), respectively, proving
\cref{eq:lower-cell-masses} under the uniform context law.

The centers of distinct context cells differ by at least \(2h\) in one
coordinate, whereas every \(B_j\) has side length \(h\).  Hence their
supports are disjoint.  Likewise, adjacent price centers are \(2w\) apart
and each price support has width \(w\), so the sets \(\mathcal S_k\) are
pairwise disjoint.  The first three bounds in
\cref{eq:lower-packing-bounds} now follow from
\(0\le\Psi,\vartheta\le1\), the disjoint context supports, and the fixed
finite quantities \(\norm{\vartheta'}_\infty\) and
\(\norm{\vartheta''}_\infty\).

It remains to verify the context H\"older bound, including pairs of points
in different cells.  Since \(\Psi\) is smooth and compactly supported,
\begin{equation}
  |\psi_j(x)-\psi_j(x')|
  \le C\min\{1,h^{-1}\norm{x-x'}\}
  \le C h^{-\alpha}\norm{x-x'}^\alpha.
\label{eq:lower-single-context-bump}
\end{equation}
When \(x,x'\) lie in the same cell, or one of them lies outside every
cell, this proves the claim directly.  When they lie in distinct cells,
the separation of the supports gives \(\norm{x-x'}\ge h\). Boundedness by
one then gives the same inequality.  Multiplication by \(\Delta\) proves
the last bound in \cref{eq:lower-packing-bounds}.  Finally,
\(\psi_j=1\) on \(B_j^\circ\), all other context bumps vanish there, and
\(\max_p\varphi_k(p)=1\).  This proves
\cref{eq:lower-core-identity} and the final assertion.
\end{proof}

\begin{proof}[Proof of \cref{lem:lower-perturbation-stability}]
We verify \cref{ass:context,ass:holder,ass:inventory} uniformly over the
packing, and \cref{ass:smooth-price} for the smooth-price construction.

\paragraph{Context arrivals.}
The context distribution is uniform on \([0,1]^d\), so
the i.i.d.{} requirement in \cref{ass:context} holds.

\paragraph{Common perturbation bound and slack choice.}
By \cref{prop:lower-packing-properties}, \(|g_\omega(x,p)|\le\Delta\),
and, with
\(\delta\lambda_\omega=g_\omega/p\),
\begin{equation}
  \norm{\delta\lambda_\omega}_\infty
  \le A_0\Delta,
  \qquad A_0:=p_{\min}^{-1}.
\label{eq:lower-uniform-location-perturbation}
\end{equation}
Fix
\begin{equation}
  \Delta_0
  \le
  \min\left\{
    \frac{s_0}{A_0},\ s_0,\ 1
  \right\},
\label{eq:lower-explicit-delta0}
\end{equation}
Choose \(c_\Delta>0\) small enough that the context H\"older constant below
is at most the prescribed \(L_x\).  Then choose \(h_0\le1/4\) so that,
for every \(h\le h_0\),
\begin{equation}
  \Delta=c_\Delta h^\alpha\le\Delta_0,
  \qquad
  \sqrt\Delta\le L_0/4.
\label{eq:lower-h0-choice}
\end{equation}
This also guarantees \(w\le L_0/4\) in both price-regularity regimes, so
every object in \cref{def:lower-packing} is well defined.

\paragraph{Location surface and nonnegative demand.}
The baseline margin and
\cref{eq:lower-uniform-location-perturbation,eq:lower-explicit-delta0} give
\[
  \lambda_\omega(x,p)+\underline z
  \ge4s_0-A_0\Delta
  \ge3s_0>0.
\]
The same inequality and compactness give uniform boundedness of demand.

The context bound in \cref{eq:lower-packing-bounds}, together with
\(p\ge p_{\min}\) and \(\Delta=c_\Delta h^\alpha\), gives
\[
  |\lambda_\omega(x,p)-\lambda_\omega(x',p)|
  \le
  Cc_\Delta\norm{x-x'}^\alpha.
\]
For price regularity of the demand surface,
\[
  \partial_p\left(\frac{g_\omega(x,p)}p\right)
  =
  \frac{\partial_pg_\omega(x,p)}p
  -
  \frac{g_\omega(x,p)}{p^2}.
\]
By \cref{eq:lower-packing-bounds}, the Lipschitz packing has
\(|\partial_pg_\omega|\le C\Delta/w\le C\) because \(w=\Delta\).
In the smooth packing,
\(w=\sqrt\Delta\), and the same first derivative is even
\(O(\sqrt\Delta)\).  Combining these bounds with the fixed smooth baseline
proves the price-Lipschitz part of \cref{ass:holder} uniformly over the two
families.

\paragraph{Residual distribution and quantile regularity.}
Set \(F_x\equiv F_0\) for every context and every packing instance.  This
constant family has H\"older modulus zero.  Its mean-zero normalization,
compact support, bounds \(K,\kappa,r_q\), and quantiles
\(z_\rho^0,z_p^0\) are identical across contexts and across the packing.  Thus
the residual-law and density conditions in
\cref{ass:context,ass:inventory} hold, and the noise part of
\cref{ass:holder} holds with \(L_F=0\). Shocks are drawn freshly from
\(F_0\) after each action. The additional slack in
\cref{lem:lower-baseline} gives
\[
  z_p^0-z_\rho^0\ge m_q+4s_0\ge m_q
\]
uniformly in \(p\), which will be used for inventory visibility below.

\paragraph{Inventory feasibility.}
The stockout
probability is
\[
  S_{\omega,x}(y\mid p)
  =
  1-F_0(y-\lambda_\omega(x,p)).
\]
Its root is
\[
  y_{\omega,x}^*(p)=\lambda_\omega(x,p)+z_p^0,
\]
and the calibration quantile is
\[
  q_{\rho,\omega}(x,p)=\lambda_\omega(x,p)+z_\rho^0.
\]
Both quantities shift from their baseline values by the same
\(\delta\lambda_\omega(x,p)\).  Hence the explicit slack choice gives
\begin{equation}
\begin{aligned}
  q_{\rho,\omega}(x,p)-(y_{\min}+r_y)
  &\ge4s_0-A_0\Delta\ge3s_0,\\
  (y_{\max}-r_y)-y_{\omega,x}^*(p)
  &\ge4s_0-A_0\Delta\ge3s_0,\\
  y_{\omega,x}^*(p)-q_{\rho,\omega}(x,p)
  &=z_p^0-z_\rho^0\ge m_q+4s_0.
\end{aligned}
\label{eq:lower-preserved-inventory-margins}
\end{equation}
These inequalities prove \cref{ass:inventory}.  The associated uniqueness,
local slope, and quadratic inventory-loss properties then follow from
\cref{lem:unique-inventory,lem:inventory-consequences}.

\paragraph{Optimal-price interiority.}
By \cref{prop:location-separation}, for every price,
\[
  G_{\omega,x}(p)
  =
  p\lambda_\omega(x,p)+C_{F_0}(p)
  =
  G_0(p)+g_\omega(x,p).
\]
All price bumps are supported in \(J\Subset\operatorname{int}(\cP_0)\).
For a context with positive bump amplitude, a point on the corresponding
price-bump plateau attains \(C_0+\Delta\psi_j(x)>C_0\), whereas every
price outside \(\cP_0\) has value at most \(C_0\).  With zero bump amplitude,
every point of \(\cP_0\), hence some interior point, is a baseline maximizer.
Thus every perturbed instance has
an optimal price in the required interior interval, and no price-monotonicity
property is used.

\paragraph{Price-value regularity.}
In the Lipschitz construction,
\cref{eq:lower-packing-bounds} gives
\(|\partial_pg_\omega|\le C\Delta/w=O(1)\), yielding the uniform
Lipschitz value bound after adding the fixed smooth baseline \(G_0\).  In the
smooth construction, \(w=\sqrt\Delta\), so
\[
  |\partial_{pp}g_\omega(x,p)|
  \le
  C\Delta/w^2
  =
  O(1),
\]
which proves the bounded-second-derivative specialization.  Outside
\(\cP_0\), the perturbation vanishes and \(G_0\) has fixed bounded
derivatives.  This completes the uniform verification of
\cref{ass:context,ass:holder,ass:inventory}, together with
\cref{ass:smooth-price} for the smooth-price packing.
\end{proof}

\begin{lemma}[Location-family information bound]
\label{lem:location-information}
With \(H^2(P,Q)=\frac12\int(\sqrt{dP}-\sqrt{dQ})^2\), shifts of the
baseline density satisfy
\[
  H^2(f_0(\cdot-u),f_0(\cdot-v))
  \le \frac{I_H}{4}|u-v|^2.
\]
Thus neighboring bump instances have one-round squared Hellinger distance
\(O(\Delta^2)\) on the affected context--price region and zero elsewhere.
The same bound holds for censored sales under every adaptive inventory action.
\end{lemma}

\begin{proof}
Let \(s=\sqrt{f_0}\).  The fundamental theorem of calculus gives
\[
  H^2(f_0(\cdot-u),f_0(\cdot-v))
  \le \frac12|u-v|^2\|s'\|_2^2
  =\frac{I_H}{8}|u-v|^2.
\]
For the neighbors used below, \(\omega\) and \(\omega'\) agree outside one
cell \(j\), while their labels in that cell are \(0\) and \(k\).  By
\cref{eq:lower-gomega-definition}, their locations agree unless
\(x\in B_j\) and \(p\in\mathcal S_k\), and everywhere differ by at most
\(|g_\omega(x,p)-g_{\omega'}(x,p)|/p_{\min}\le C\Delta\).  Censoring is a
measurable image of demand, so data processing proves the claim.
\end{proof}

\section{Complete Direct-Sum Proof of the Lower Bound}
\label{app:lower-complete}

We use the following sequential form of squared Hellinger subadditivity.
Let \(P_t(\cdot\mid z_{<t})\) and \(Q_t(\cdot\mid z_{<t})\) be
probability kernels on standard Borel spaces, and let \(P,Q\) be their
induced laws on \(Z_1,\ldots,Z_n\).  With the normalization in
\cref{lem:location-information},
\begin{equation}
  H^2(P,Q)\le7\E_P\!\left[
    \sum_{t=1}^n
    H^2\!\left(P_t(\cdot\mid Z_{<t}),Q_t(\cdot\mid Z_{<t})\right)
  \right].
\label{eq:sequential-hellinger-chain}
\end{equation}
This is Lemma~F.2 of \citet{foster2024online}.
The inequality permits adaptive kernels and an expectation under just one
reference law.  Its absolute factor is absorbed into the constants below.

\begin{proof}[Complete proof of \cref{thm:lower}]
Fix one of the two price-regularity regimes, set \(w=\Delta\) in the
Lipschitz regime and \(w=\sqrt\Delta\) in the twice-smooth regime, and use
the family \(\Omega\) in \cref{def:lower-packing}.  By
\cref{lem:lower-perturbation-stability}, every member of this family belongs
to the claimed model class.  It remains to show that no adaptive
policy can learn all context-cell labels quickly enough.  We give the
argument for a deterministic policy. Conditioning on a randomized policy's
internal seed and then applying Yao's principle gives the stated conclusion
for randomized policies.

All objects below are those formally constructed in
\cref{def:lower-packing}. In particular, the price supports
\(\mathcal S_k\) are pairwise disjoint and \(\psi_j=1\) on
\(B_j^\circ\).  We use the three counts
\begin{equation}
\begin{aligned}
  \overline N_{j,k}
  &=\sum_{t=1}^T
    \one\{x_t\in B_j,\ p_t\in\mathcal S_k\},\\
  N_j^\circ
  &=\sum_{t=1}^T\one\{x_t\in B_j^\circ\},\\
  N_{j,k}^\circ
  &=\sum_{t=1}^T
    \one\{x_t\in B_j^\circ,\ p_t\in\mathcal S_k\}.
\end{aligned}
\label{eq:lower-counts}
\end{equation}
The first count includes the transition shell
\(B_j\setminus B_j^\circ\).  This inclusion is essential because
observations in that shell also carry information about the label.  The core
counts will be used only to lower-bound regret.

Fix a cell \(j\).  Put an independent uniform prior on the labels of all
other cells.  Let \(\Pbb^{(j,0)}\) be the resulting mixture history law when
cell \(j\) has no bump, and let \(\Pbb^{(j,k)}\) be the mixture law when its
label is \(k\).  The no-bump instance serves as the reference experiment.  The final prior
is supported on the positive-label instances.  Conditional on a fixed
configuration of the other labels,
\cref{eq:sequential-hellinger-chain,lem:location-information} give
\begin{equation}
  H^2(\Pbb^{(j,0)}_{\omega_{-j}},
      \Pbb^{(j,k)}_{\omega_{-j}})
  \le
  C\Delta^2
  \E^{(j,0)}_{\omega_{-j}}[\overline N_{j,k}].
\label{eq:lower-sequential-hellinger}
\end{equation}
Apply this inequality to the sequential variables
\[
  (x_1,(p_1,y_1),O_1,\ldots,x_T,(p_T,y_T),O_T).
\]
The context kernels
are the same uniform law in both experiments.  Conditional on a common
history and the current context, the policy also uses the same action
kernel.  Their Hellinger contributions in
\cref{eq:sequential-hellinger-chain} are therefore zero.  For the
observation step, conditional on a common history and action, the two
locations differ by
\(\Delta\psi_j(x_t)\varphi_k(p_t)/p_t\).  Thus
\cref{lem:location-information} bounds that step's squared Hellinger
distance by
\[
  C\Delta^2\one\{x_t\in B_j,\ p_t\in\mathcal S_k\}.
\]
This bound includes every feasible inventory: the same censoring map
\(D\mapsto\min\{D,y_t\}\) is applied in the two conditional experiments.
Taking \(P\) in \cref{eq:sequential-hellinger-chain} to be the zero-label
reference law, summing, and absorbing the factor \(7\) proves
\cref{eq:lower-sequential-hellinger}.  The history spaces here are finite
products of Euclidean spaces, so the required conditional kernels exist.

Squared Hellinger distance is jointly convex.  Averaging
\cref{eq:lower-sequential-hellinger} over the other labels and then over
\(k\) gives
\begin{equation}
\begin{aligned}
  \frac1{K_\Delta}\sum_{k=1}^{K_\Delta}
  H^2(\Pbb^{(j,0)},\Pbb^{(j,k)})
  &\le
  \frac{C\Delta^2}{K_\Delta}
  \E^{(j,0)}\!\left[
    \sum_{k=1}^{K_\Delta}\overline N_{j,k}
  \right]\\
  &\le
  C\frac{Th^d\Delta^2}{K_\Delta}.
\end{aligned}
\label{eq:lower-average-information}
\end{equation}
The second inequality uses disjointness of the price supports and
\(\Pbb(x_t\in B_j)=h^d\), as established in
\cref{prop:lower-packing-properties}.  Choose the fixed packing constant
\(c_\Delta\) and the bandwidth so that
\begin{equation}
  \frac{Th^d\Delta^2}{K_\Delta}\le c_{\rm inf}
\label{eq:lower-information-balance}
\end{equation}
for a sufficiently small numerical \(c_{\rm inf}\).  Since
\(\operatorname{TV}(P,Q)\le\sqrt{2H^2(P,Q)}\), Jensen's inequality and
\cref{eq:lower-average-information} yield
\begin{equation}
  \frac1{K_\Delta}\sum_{k=1}^{K_\Delta}
  \operatorname{TV}(\Pbb^{(j,0)},\Pbb^{(j,k)})
  \le c_{\rm tv},
\label{eq:lower-average-tv}
\end{equation}
where \(c_{\rm tv}\) can be made smaller than \(1/8\).

We next convert indistinguishability into regret.  Define
\[
  A_{j,k}=\{N_j^\circ>0,\ N_{j,k}^\circ\ge N_j^\circ/2\},
  \qquad
  E_{j,k}=A_{j,k}^c
  \cap\{N_j^\circ\ge c_BTh^d\}.
\]
Because the price supports are disjoint, on every sample path with
\(N_j^\circ>0\), at most two of the events \(A_{j,k}\) can occur.  Therefore
\begin{equation}
  \frac1{K_\Delta}\sum_{k=1}^{K_\Delta}
  \Pbb^{(j,0)}(A_{j,k})
  \le\frac2{K_\Delta}.
\label{eq:lower-null-allocation}
\end{equation}
By \cref{eq:lower-cell-masses}, the random variable \(N_j^\circ\) is
binomial with mean \(2^{-d}Th^d\), independently of the policy.  For the
bandwidths chosen at the end of the proof, this mean diverges:
\(Th^d\asymp T^{3\alpha/(d+3\alpha)}\) in the Lipschitz regime and
\(Th^d\asymp T^{5\alpha/(2d+5\alpha)}\) in the smooth regime.
Taking \(c_B=2^{-d-1}\), the binomial lower-tail bound gives
\[
  \Pbb^{(j,0)}(N_j^\circ<c_BTh^d)
  \le\exp\{-2^{-d}Th^d/8\}.
\]
In particular, for all sufficiently large \(T\),
\begin{equation}
  \Pbb^{(j,0)}(N_j^\circ\ge c_BTh^d)\ge\frac34.
\label{eq:lower-context-count}
\end{equation}
Equations
\cref{eq:lower-average-tv,eq:lower-null-allocation,eq:lower-context-count}
imply
\begin{equation}
  \frac1{K_\Delta}\sum_{k=1}^{K_\Delta}
  \Pbb^{(j,k)}(E_{j,k})
  \ge\frac34-\frac2{K_\Delta}-c_{\rm tv}
  \ge\frac38=:c_E.
\label{eq:lower-error-probability}
\end{equation}
Indeed, the average reference probability is at least
\(3/4-2/K_\Delta\), and changing from the reference to the alternatives
reduces it by at most the average total-variation distance.  Here we take
\(T\) large enough that \(K_\Delta\ge8\) and use
\(c_{\rm tv}\le1/8\).

By \cref{eq:lower-core-identity}, on \(B_j^\circ\) under label
\(\omega_j=k\) and for \(p\in\cP_0\), the optimized value is
\[
  G_{\omega,x}(p)=C_0+\Delta\varphi_k(p).
\]
Its maximum is \(C_0+\Delta\).  Every price in \(\cP_0\) but outside
\(\mathcal S_k\) has gap exactly \(\Delta\). Every price outside
\(\cP_0\) has gap at least \(\Delta\), because the baseline extension is at
most \(C_0\) and all price bumps are supported in \(\cP_0\).  On
\(E_{j,k}\), at least \(N_j^\circ/2\) core visits use prices outside
\(\mathcal S_k\).  Inventory suboptimality is nonnegative, so
\[
\begin{aligned}
  \Reg_j
  &:=\sum_{t:x_t\in B_j^\circ}
    \{V_{x_t}-Q_{x_t}(p_t,y_t)\}\\
  &\ge
    \frac{\Delta}{2}N_j^\circ
    \one\{E_{j,k}\}.
\end{aligned}
\]
Combining this inequality with
\cref{eq:lower-error-probability} proves the cellwise Bayes bound
\begin{equation}
  \frac1{K_\Delta}\sum_{k=1}^{K_\Delta}
  \E^{(j,k)}[\Reg_j]
  \ge cTh^d\Delta.
\label{eq:lower-cell-bayes}
\end{equation}

Now draw \(\omega\) uniformly from
\(\Omega=\{1,\ldots,K_\Delta\}^{M_h}\). Equivalently, all cell labels are
independent and uniform.  For fixed \(j\), integrating out the other labels
gives exactly the mixture experiment above.  The conditional observation
kernel can depend directly on \(\omega_j\) only on rounds with
\(x_t\in B_j\), which is precisely the localization used in
\cref{eq:lower-sequential-hellinger}.  The core cells are disjoint, so summing
\cref{eq:lower-cell-bayes} yields Bayes regret
\begin{equation}
  \E_{\omega\sim\operatorname{Unif}(\Omega)}[\Reg_T]
  \ge
  cM_hTh^d\Delta
  \ge cT\Delta.
\label{eq:lower-direct-sum}
\end{equation}
Yao's principle turns \cref{eq:lower-direct-sum} into a worst-instance lower
bound for every randomized policy.

It remains to solve the information balance.  In the Lipschitz construction,
\(w=\Delta\) and \(K_\Delta\asymp\Delta^{-1}\), so
\cref{eq:lower-information-balance} is tight at
\[
  Th^d\Delta^3\asymp1.
\]
Together with \(\Delta=c_\Delta h^\alpha\), this gives
\[
  h\asymp T^{-1/(d+3\alpha)},
  \qquad
  T\Delta\asymp
  T^{(d+2\alpha)/(d+3\alpha)}.
\]
In the twice-smooth construction,
\(w=\sqrt\Delta\) and
\(K_\Delta\asymp\Delta^{-1/2}\).  Hence
\[
  Th^d\Delta^{5/2}\asymp1,
  \qquad
  h\asymp T^{-2/(2d+5\alpha)},
\]
and
\[
  T\Delta\asymp
  T^{(2d+3\alpha)/(2d+5\alpha)}.
\]
This proves both lower bounds for all \(T\ge T_0\), with a fixed
\(T_0\).  To cover the remaining horizons, fix one admissible packing
with \(K_\Delta\ge2\) and positive, fixed \(h,\Delta\).  Before the
first observation, the independent uniform label of a core cell is unknown
to the policy.  Conditional on any first-round context in that core, the
probability of selecting its correct price support, averaged over labels,
is at most \(1/K_\Delta\).  Consequently the first-round Bayes regret is
at least
\[
  c_*:=\Delta(1-1/K_\Delta)M_h2^{-d}h^d>0,
\]
uniformly over policies.  Subsequent regret is nonnegative.  For either
claimed exponent \(r\), reducing the lower-bound constant to at most
\(c_*/T_0^r\) therefore covers every integer \(1\le T<T_0\).
\end{proof}

\section{Experimental Details}
\label{app:experiment-details}

\subsection{Additional Details for the Main Experiment}
\label{app:main-experiment-details}

\paragraph{Demand and noise environment.}
We generate latent demand from the additive location model
\[
  D_t=\lambda_{\alpha_0}(x_t,p_t)+\epsilon_t,
  \qquad
  \epsilon_t\sim\operatorname{Unif}[-a,a].
\]
To construct an \(\alpha_0\)-Hölder-smooth demand surface, define the centered cusp
\[
 \psi_{c,\alpha_0}(u)
 =|u-c|^{\alpha_0}
  -\frac{c^{\alpha_0+1}+(1-c)^{\alpha_0+1}}{\alpha_0+1}.
\]
For $d=1$, the location surface is
\[
\lambda_{\alpha_0}(x,p)
=1.55-0.58\psi_{0.35,\alpha_0}(x)
 +0.22\psi_{0.78,\alpha_0}(x)-0.42p
 +0.16p\psi_{0.62,\alpha_0}(x).
\]
For $d=2$, it is
\begin{align*}
\lambda_{\alpha_0}(x,p)={}&1.55
-0.42\psi_{0.32,\alpha_0}(x_1)
+0.28\psi_{0.79,\alpha_0}(x_1)\\
&-0.36\psi_{0.41,\alpha_0}(x_2)
+0.24\psi_{0.73,\alpha_0}(x_2)-0.42p\\
&+0.10p\psi_{0.61,\alpha_0}(x_1)
-0.08p\psi_{0.27,\alpha_0}(x_2).
\end{align*}

Each surface is $\alpha_0$-H\"older and contains an $\alpha_0$-order cusp. The chosen action produces censored sales $O_t=\min\{D_t,y_t\}$.
The learner observes only $O_t$.

\paragraph{Parameter choices and data generation.}
The action domains are $p\in[0.1,1]$ and $y\in[0,2.2]$.  We set the shortage
cost to $b=0.4$, the holding cost to $h_c=1$, the visible quantile level to
$\rho=0.2$, and the noise half-width to $a=0.5$.  The scaling experiment uses
$T\in\{4000,8000,16000,32000,64000\}$ and five repetitions for each configuration
$(d,\alpha_0)\in \{(1,0.5),(1,0.75),(1,1),(2,0.5),(2,0.75),(2,1)\}$.

For each repetition, contexts are drawn i.i.d. uniformly from $[0,1]^d$ and
residual shocks are drawn independently from $\operatorname{Unif}[-0.5,0.5]$.
The complete context and residual-shock arrays are generated before the policy starts, using independent child random streams.

For replication \(r\) under configuration \((d,\alpha_0,T)\), the random seed is set to \(20260824+7919d+1000\alpha_0+9176T+1009r\).

The algorithm is supplied with the true environmental smoothness,
$\widehat\alpha=\alpha_0$.  At declared horizon $T$, its bandwidth and
profiling parameters are

\[
  h=\left(\frac{\log(eT)}{T}\right)^{1/(d+3\alpha)},
  \qquad
  \varepsilon=h^\alpha,
  \qquad
  \delta_x=\varepsilon^{1/(2\alpha)}.
\]

Here $\alpha=\widehat\alpha=\alpha_0$.  The decision-price mesh is
$\varepsilon$, the profiler-anchor mesh is $\sqrt\varepsilon$, and profiler
cells have diameter at most $\delta_x$.  Each midpoint threshold uses a
batch of

\[
  m_y=\left\lceil C_{\rm tune}\eta^{-2}
  \log(T|\Gamma_y||\mathcal Z_{\rm prof}|B_y/T^{-2})
  \right\rceil
\]

observations, where $\eta=c_\eta\sqrt{\varepsilon}$.  We set $C_{\rm tune}=0.05$, $c_\eta=0.25$, and $h_{\rm vis}=1$.
$C_{\rm tune}$ is a fixed scale factor in batch lengths and confidence radii.
The reported values are chosen empirically.  The theoretical confidence
guarantees require the tuning conditions in \cref{app:tuning}.  The value $h_{\rm vis}$ is only a
technical startup cutoff.  Every experimental bandwidth satisfies $h\le1$,
so this cutoff never changes an action in the reported experiments.

\paragraph{Simulation protocol and evaluation.}
We run the Lipschitz-price mode of \cref{alg:mck}.  Within each unfinished
profiler cell, anchors are visited round-robin.  A midpoint inventory threshold
is held fixed for $m_y$ observations and compared with
$\tau(p)\pm2\eta$.  A decision price uses the completed inventory
estimate of its nearest anchor.  Completed cells switch immediately to
calibrated kernel UCB, whereas unfinished cells continue profiling on their own
arrivals.  Only UCB rounds enter the price-specific proxy and sales-quantile
histories.

The regret benchmark maximizes $G_x(p)$ on an 801-point price grid over
$[0.1,1]$.  One-round regret is the expected-value gap
\[
  G_{x_t}^*-Q_{x_t}(p_t,y_t),
\]
.  Reported curves are means across
repetitions. Shaded bands are mean plus or minus $1.96$ standard errors.
The top row of \cref{fig:scaling-d-alpha} reports $d=1$ and the bottom row
reports $d=2$.  The dashed guide is proportional to $T^{\beta(d,\alpha)}$
and is normalized at the first plotted horizon.

\subsection{Twice-Smooth Price Experiment}
\label{app:twice-smooth-experiment}

We evaluate the twice-smooth price specialization of
\cref{cor:smooth-upper} using the controlled demand environment, six
$(d,\alpha_0)$ configurations, horizons, and oracle regret benchmark from
\cref{app:main-experiment-details}.  We run five trajectories
for each configuration and horizon, using base seed $2123345$ and the same configuration-dependent seed
offsets as in the main experiment.
For the controlled uniform-noise model, the optimized price-value curve is
twice continuously differentiable on the compact price interval, with a
uniformly bounded second derivative, so this environment belongs to the
specialized class in \cref{ass:smooth-price}.

With $\alpha=\widehat\alpha=\alpha_0$, the twice-smooth mode sets
\[
  h=\left(\frac{\log(eT)}{T}\right)^{2/(2d+5\alpha)},
  \qquad
  \varepsilon=h^\alpha,
  \qquad
  \Delta_p=\sqrt{\varepsilon}.
\]
The profiler retains context resolution
$\delta_x=\varepsilon^{1/(2\alpha)}$, price-anchor mesh
$\sqrt\varepsilon$, and the same asynchronous profiling protocol.  Hence the
decision grid and profiler-anchor grid coincide in this mode.  We use $C_{\rm tune}=0.01$, $c_\eta=0.3$, and $h_{\rm vis}=1$.
The first two constants differ from those in the main experiment, so the
comparison reflects changes in both the price grid and tuning constants.
The dashed guide in each panel is proportional to
\[
  T^{\beta_{\rm sm}(d,\alpha)},
  \qquad
  \beta_{\rm sm}(d,\alpha)
  =\frac{2d+3\alpha}{2d+5\alpha}.
\]

\begin{figure}[H]
  \centering
  \includegraphics[width=\textwidth]{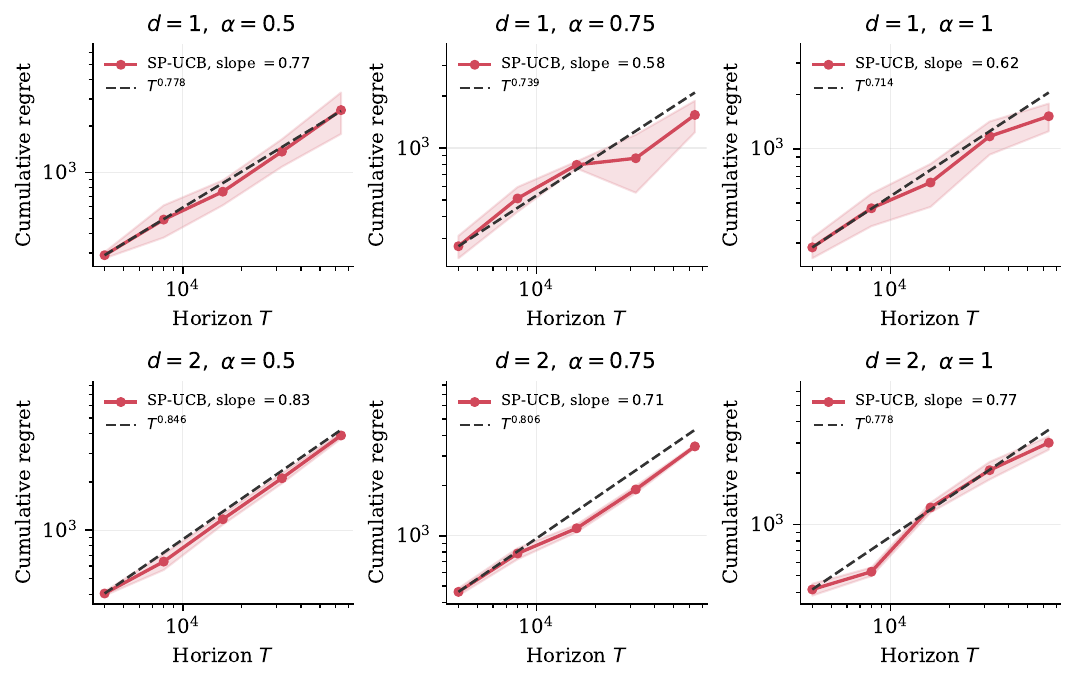}
  \caption{Cumulative-regret scaling in the twice-smooth price mode.  Curves are
  means over five trajectories. Dashed curves are representative $T^{\beta_{\rm sm}}$
  guides normalized at the first plotted horizon.}
  \label{fig:twice-smooth-scaling}
\end{figure}

\subsection{Comparisons with Neighboring Online Methods}
\label{app:neighboring-method-comparisons}

We compare with two online joint pricing--inventory methods developed for a
repeated, noncontextual demand environment: the CWZ trisection method of
\citet{chen2024optimal} and a SALA-style implementation based on
\citet{chen2021nonparametric}.  We construct context-agnostic baselines by
pooling observations across contexts and applying CWZ and SALA to the
resulting noncontextual problem.  The contextual method of
\citet{han2026pricingstock} is discussed in \cref{sec:related}.

\paragraph{Demand and residual environment.}
Contexts are drawn i.i.d. uniformly from $[0,1]^2$.  Define
\[
 s(x)=0.25x_1+0.25x_2+0.5x_1x_2,
 \qquad w(x)=0.70+0.45s(x).
\]
Demand follows
\[
 D_t=\lambda(x_t,p_t)+Z_t,
 \qquad
 \lambda(x,p)=2.15+\frac{1.70}{1+2\exp(16(p-w(x)))},
 \qquad Z_t\sim\operatorname{Unif}[-0.15,0.15].
\]
The residuals are independent of context and price and have mean zero.
The interaction term $x_1x_2$ shifts the location of the smooth demand
response to price.  Demand decreases strictly with price, while both the
contextual response and the price response are nonlinear.  Since
$2.15<\lambda(x,p)<3.85$, realized demand lies in $[2,4]$ without clipping.
The comparison evaluates context-pooled implementations on this contextual
instance.

\paragraph{Parameter choices and data generation.}
All methods use $p\in[0.1,1]$ and $y\in[0,5]$.  The shortage cost is $b=0.4$,
the holding cost is $h_c=1$, the visible quantile level is $\rho=0.2$, and
$\alpha=1$.  We use $T=16000$ and five independent repetitions.
For this comparison, SP-UCB uses $C_{\rm tune}=0.01$ and a profiling
precision multiplier $c_\eta=0.4$:
\[
 h=\varepsilon=0.2317,\qquad \delta_x=\sqrt{\varepsilon}=0.4813,
 \qquad \eta=c_\eta\sqrt{\varepsilon}=0.1925.
\]
The context partition and profiler anchor mesh retain their original
resolution.  There are nine profiling cells, five decision prices, and
three profiler anchors.  The stopping thresholds use $\tau(p)\pm2\eta$,
and the batch size is computed using the same $\eta$, giving ten observations
per midpoint threshold.  The multiplier $C_{\rm tune}$ is used in the
profiling batch size, the interval-width stopping condition, and the UCB
bonus.  These are empirical tuning choices.

The base seed is $20260827$. Repetition $r\in\{0,\ldots,4\}$ uses seed
$20260827+9176T+1009r$.  Independent child random streams generate the
complete context array and uniform residual array before any policy starts.
Within each repetition, all three methods receive the same $(x_t,Z_t)$ on
each calendar round.

\paragraph{Comparator implementations and evaluation.}
CWZ uses minimum epoch size $8$, sample scale $0.006$, and final exploitation
fraction $0.35$.  The SALA-style implementation explores seven prices with
exploration scale $1.1$, estimates a pooled demand curve and residual law,
and then exploits one price--inventory pair.  Both comparators pool
observations across contexts.  The regret benchmark maximizes $G_x(p)$
on an 801-point price grid over $[0.1,1]$.  Curves report means over the five repetitions, with shaded bands showing
$\pm1.96$ standard errors.

\begin{figure}[H]
 \centering
 \includegraphics[width=\textwidth]{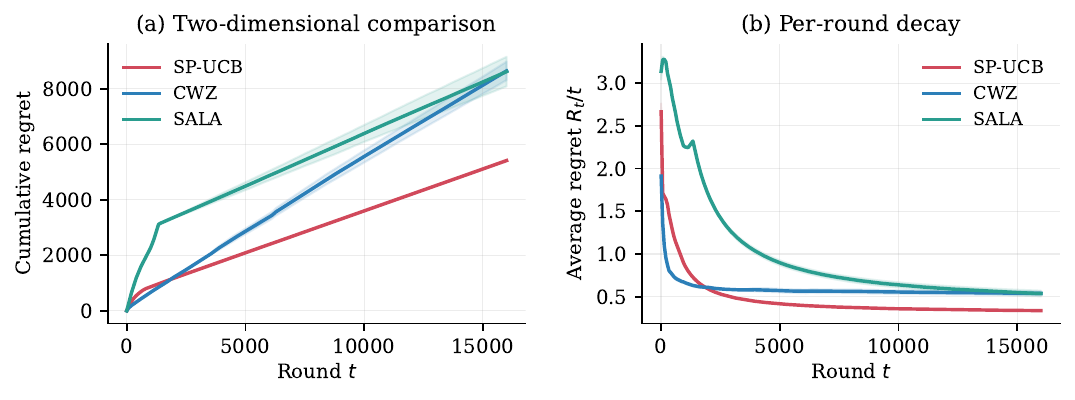}
 \caption{Comparison on the nonlinear interaction instance with
 $y\in[0,5]$, $T=16000$, $C_{\rm tune}=0.01$, and $c_\eta=0.4$.
 Curves are means over five common-random-number repetitions.}
 \label{fig:comparison-2d}
\end{figure}

In \cref{fig:comparison-2d}, SALA incurs substantial initial regret while exploring prices with high inventory to reconstruct demand, followed by approximately linear accumulation after committing to a fixed price–inventory pair. SP-UCB also incurs initial inventory-profiling costs, but subsequent contextual UCB learning yields a lower empirical regret slope. CWZ maintains a higher slope throughout much of the horizon: its search phase consumes 65\% of the rounds, while its final policy retains a substantial per-round gap. This behavior partly reflects the current implementation, which defaults to midpoint inventory when the search budget is exhausted.

\subsection{Smoothness-Mismatch Experiment}
\label{app:smoothness-mismatch}

We use the same cusp demand family and simulation protocol as in
\cref{app:main-experiment-details}, but decouple the environmental smoothness
from the smoothness supplied to the algorithm.  The experiment crosses
\[
  (d,\alpha_0)\in
  \{1,2\}\times\{0.5,0.75,1\},
  \qquad
  \widehat\alpha\in\{0.5,0.75,1\},
\]
over horizons \(T\in\{4000,8000,16000,32000\}\), with one repetition for
each configuration.  The environment uses the true cusp exponent
\(\alpha_0\), whereas the algorithm constructs its bandwidth, price mesh,
profile mesh, and confidence radii using \(\widehat\alpha\).  All other inputs
are held fixed: we use the Lipschitz-price mode, \(C_{\rm tune}=0.05\), \(c_\eta=0.3\),
\(h_{\rm vis}=1\), visible quantile level \(\rho=0.2\), shortage cost
\(b=0.4\), holding cost \(h_c=1\), and noise half-width \(a=0.5\).

For fixed \((d,\alpha_0,T,r)\), all three values of \(\widehat\alpha\) receive
the same pre-generated context and residual-shock arrays.  The complete
context and residual-shock arrays are generated before the policy starts using
independent child random streams, as in the main scaling experiment.

For reproducibility, the mismatch base seed is $20260903$, with
the same configuration-dependent offsets as in the main experiment,
independently of \(\widehat\alpha\).
The regret benchmark is computed as in
\cref{app:main-experiment-details}.
The dashed reference curve in each panel is proportional to
\(T^{\beta(d,\alpha_0)}\), normalized at the first plotted horizon, where
\[
  \beta(d,\alpha_0)=\frac{d+2\alpha_0}{d+3\alpha_0},
\]
the exponent obtained when the true smoothness governs the rate.

\begin{figure}[H]
  \centering
  \includegraphics[width=\textwidth]{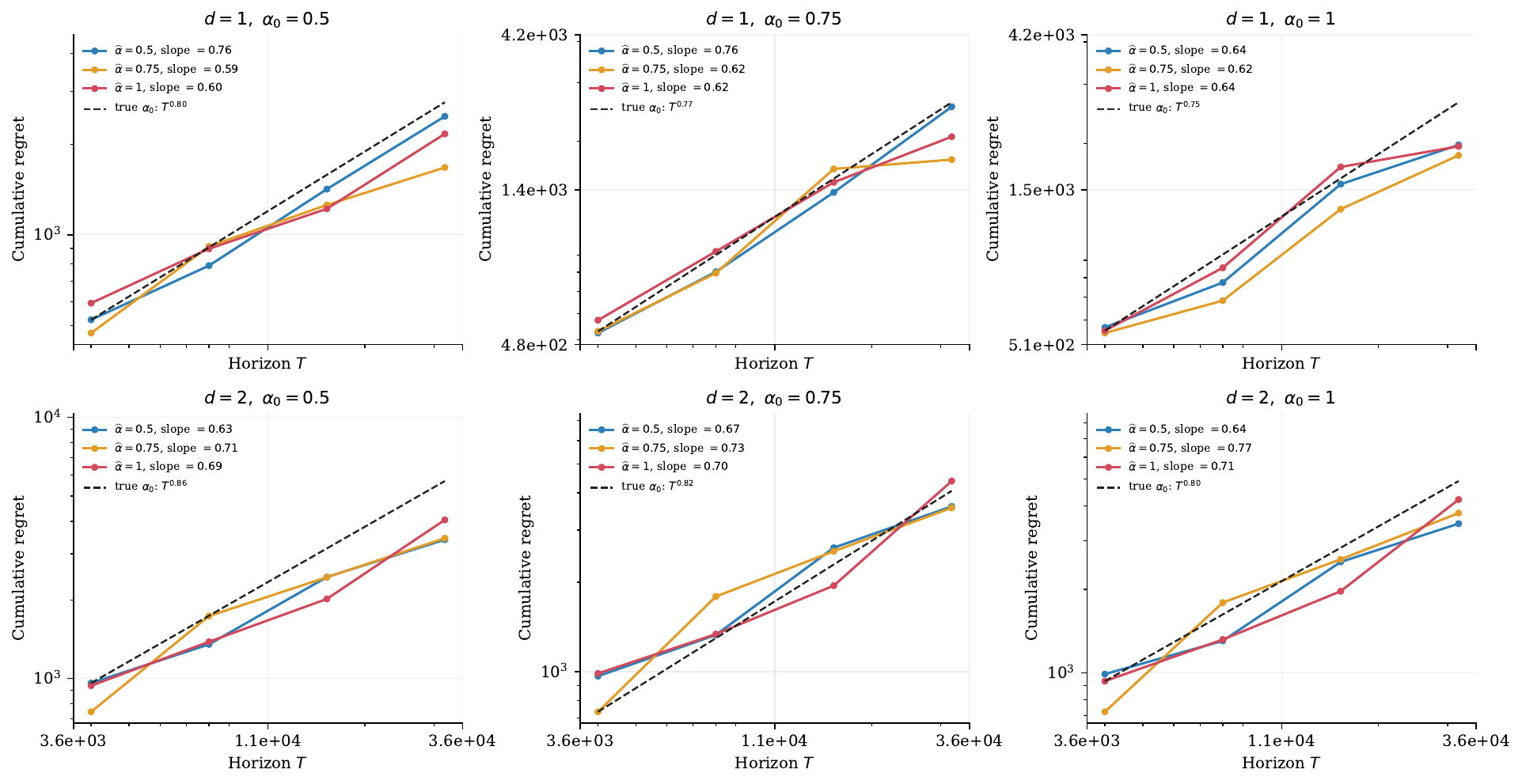}
  \caption{Scaling under smoothness-input mismatch.  Each panel fixes the
  true cusp exponent \(\alpha_0\) and dimension \(d\), and each colored curve
  uses a different algorithm input \(\widehat\alpha\).  The dashed black curve
  is a \(T^{\beta(d,\alpha_0)}\) guide computed from the true
  smoothness and normalized to the matched-input curve at $T=4000$.
  Each configuration uses one trajectory. The bands therefore have zero width.}
  \label{fig:alpha-mismatch}
\end{figure}

At $T=32000$, the input $\widehat\alpha=0.75$ has the lowest regret
in all three one-dimensional panels.  In two dimensions, the lowest-regret
inputs are $0.5$, $0.75$, and $0.5$ for true smoothness $0.5$, $0.75$, and $1$,
respectively. $\widehat\alpha=1$ has the highest regret in all three panels.
The fitted log--log slopes range from $0.59$ to $0.77$ across the 18 curves.
The lowest-regret smoothness input therefore varies with the configuration
and horizon.  These 72 runs describe finite-horizon behavior along one
trajectory per configuration.

Changing $\widehat\alpha$ changes the bandwidth, decision-price grid, and
profiling resolution simultaneously, so the observed curves reflect their
combined effects.  At the level of the theoretical model, an
$\alpha_0$-H\"older function is also $\widehat\alpha$-H\"older for
$\widehat\alpha\le\alpha_0$, with an adjusted constant on the compact domain.
Conservative smoothness inputs therefore satisfy the smoothness requirement,
with the remaining tuning conditions given in \cref{app:tuning}.
For $\widehat\alpha>\alpha_0$, the true smoothing bias may exceed the supplied
bias correction and fall outside the scope of the current analysis.

\section{Inventory Error in Price Evaluation}
\label{app:inventory-sensitivity}

Inventory-learning error affects the price evaluations in \citet{chen2024optimal} and our method differently. Their Algorithm 2 estimates profit differences between pairs of prices. Our method constructs a calibrated score for each price that can be reused across comparisons. In the setting studied below, inventory error introduces a first-order term into their pairwise correction. Our fixed low-quantile correction remains valid whenever the quantile is visible, so inventory error affects our score only through the quadratic profit loss near the optimal inventory. This allows coarse inventory estimates to support fine price comparisons.

\subsection{Population Sensitivity}
\label{app:inventory-sensitivity-theory}

To isolate the effect of inventory error on price comparisons, fix a
context \(x\) and suppress the context subscript.  Under the location model,
\(D(p)=\lambda(p)+\xi\), where \(\xi\sim F\), the fixed quantile satisfies
\(q_\rho(p)=\lambda(p)+z_\rho\).  Since
\(\E[Y(p,y,O)]=Q(p,y)+b\lambda(p)\), our population score is
\begin{equation}
 J(p,y):=\E[Y(p,y,O)]-bq_\rho(p)=Q(p,y)-bz_\rho.
 \label{eq:inventory-independent-centering}
\end{equation}
For \(y>q_\rho(p)\), the sales quantile equals the demand quantile,
so the correction is identifiable and inventory error enters \(J\) only
through the profit loss in \(Q\).

For two fixed prices \(p<p'\), let \(b,h_c>0\) and write
\[
 C_p:=p+b+h_c,\qquad \tau(p):=h_c/C_p,\qquad
 z_p:=F^{-1}(1-\tau(p)),\qquad y^*(p)=\lambda(p)+z_p.
\]
Take feasible inventories \(y=y^*(p)+e\), \(y'=y^*(p')+e'\), and define
\[
 \ell_p(e):=G(p)-Q(p,y^*(p)+e),\qquad
 L(v):=\E[(\xi-v)^+],\qquad \Delta G:=G(p')-G(p).
\]
Let \(\Delta_{\rm SP}=J(p',y')-J(p,y)\), and let
\(\Delta_{\rm C2}\) denote the population limit of CWZ2 at these same
inventories.

We compare these population scores at the prescribed feasible inventories.
The fixed-context instances also satisfy the decreasing, twice-differentiable
demand and noise-density conditions of \citet{chen2024optimal}.

\begin{proposition}[Sensitivity to inventory-learning error]
\label{prop:inventory-sensitive-comparison}
For the two-price comparison above, with visible low quantiles and
\(z_{p'}+e'>z_p\), the estimation errors satisfy, as \(e,e'\to0\),
\begin{equation}
 \Delta_{\rm SP}-\Delta G=O(e^2+e'^2),\qquad
 \Delta_{\rm C2}-\Delta G=b\tau(p)e+O(e^2+e'^2).
 \label{eq:inventory-sensitivity-orders}
\end{equation}
\end{proposition}

\begin{proof}
On the common fixed-context model, \cref{ass:context,ass:inventory}
provide bounded density and interior critical inventories.  The additional
price and density regularity imposed by \citet{chen2024optimal} places this
instance in their model class.  The identities below follow from the shared-noise
and inventory assumptions.
Visibility identifies the low quantile used in \(J\).  Define
\(A_p(e)=\int_0^e[F(z_p+t)-F(z_p)]\,dt\).
The density bound gives \(0\le A_p(e)\le\kappa'e^2/2\), for either
sign of \(e\).  Integrating
\(\partial_yQ(p,y)=C_p[F(z_p)-F(y-\lambda(p))]\) yields
\(\ell_p(e)=C_pA_p(e)\).  Together with
\cref{eq:inventory-independent-centering}, this gives
\begin{equation}
 \Delta_{\rm SP}-\Delta G=\ell_p(e)-\ell_{p'}(e').
 \label{eq:sp-exact-inventory-error}
\end{equation}

For CWZ2, put \(a=z_p+e\), \(a'=z_{p'}+e'\).  The high-price remaining
inventory is \(\delta'=(a'-\xi)^+\).  CWZ2 inverts its distribution at
\(\tau(p)\), obtaining \(u=a'-z_p>0\):
\[
 \Pbb(\delta'\le a'-z_p)=\Pbb(\xi\ge z_p)=\tau(p).
\]
The local density lower bound at \(z_p\) makes this inversion unique.
Its correction is therefore
\[
 b\left[u\tau(p)-\E[\delta'\one\{0<\delta'\le u\}]\right]
 =b[L(z_p)-L(a')].
\]
The correction needed at the actual inventories is
\(b[L(a)-L(a')]\), since
\(Q(p,y)=\E[pO-h_c(y-O)]-bL(a)\).
Their difference gives
\begin{equation}
 \Delta_{\rm C2}-\Delta G
 =\ell_p(e)-\ell_{p'}(e')+b[L(z_p)-L(z_p+e)].
 \label{eq:cwz-extra-correction}
\end{equation}
Finally, integrating \(L'(v)=-(1-F(v))\) gives
\(L(z_p)-L(z_p+e)=\tau(p)e-A_p(e)\).  Substituting this and the
quadratic bounds for \(\ell_p(e)\) and \(\ell_{p'}(e')\) proves
\cref{eq:inventory-sensitivity-orders} without requiring a differentiable
density.
\end{proof}

For price-evaluation accuracy \(\varepsilon\), SP permits inventory errors
of order \(\sqrt\varepsilon\), supporting the sharing in \cref{lem:profile}.
For fixed \(p<p'\) and \(e'=0\), CWZ2 generally requires
\(e=O(\varepsilon)\).  Sampling and context-localization errors are separate.
This comparison is specific to CWZ2: in the \(a'>a\) branch, CWZ's
Algorithm~5 has extra error \(b(a'-a)[F(a)-F(z_p)]\), whose coefficient
also shrinks with the truncation-point separation.

\subsection{Controlled Experiments}
\label{app:inventory-sensitivity-sampling}

\paragraph{Population comparison.}
Use \(D(p)=c-p+\xi\), \(\xi\sim\operatorname{Unif}[-1,1]\),
\(c=3+1/12+0.005\), \(b=h_c=1\), and \(\rho=0.2\).
The price range \([0.8,2.05]\) and inventory range \([0,4.5]\) keep
demand nonnegative and the relevant quantiles interior.  Compare
\(p=1,p'=2\), with \(e'=0\) and \(|e|\le0.04\).
The optimized value is \(G(p)=p(c-p)-(p+1)/(p+2)\), giving
\begin{equation}
 \Delta G=0.005,\qquad
 \Delta_{\rm SP}-\Delta G=\tfrac34e^2,\qquad
 \Delta_{\rm C2}-\Delta G=\tfrac13e+\tfrac12e^2.
 \label{eq:inventory-diagnostic-exact}
\end{equation}
At \(e=-0.04\), the SP and CWZ2 estimates are \(0.006200\) and
\(-0.007533\): the latter reverses the ranking despite an inventory
profit loss of only \(0.0012\).  \Cref{fig:inventory-population-sensitivity}
shows the different sensitivities without sampling noise.

\begin{figure}[H]
 \centering
 \includegraphics[width=\textwidth]{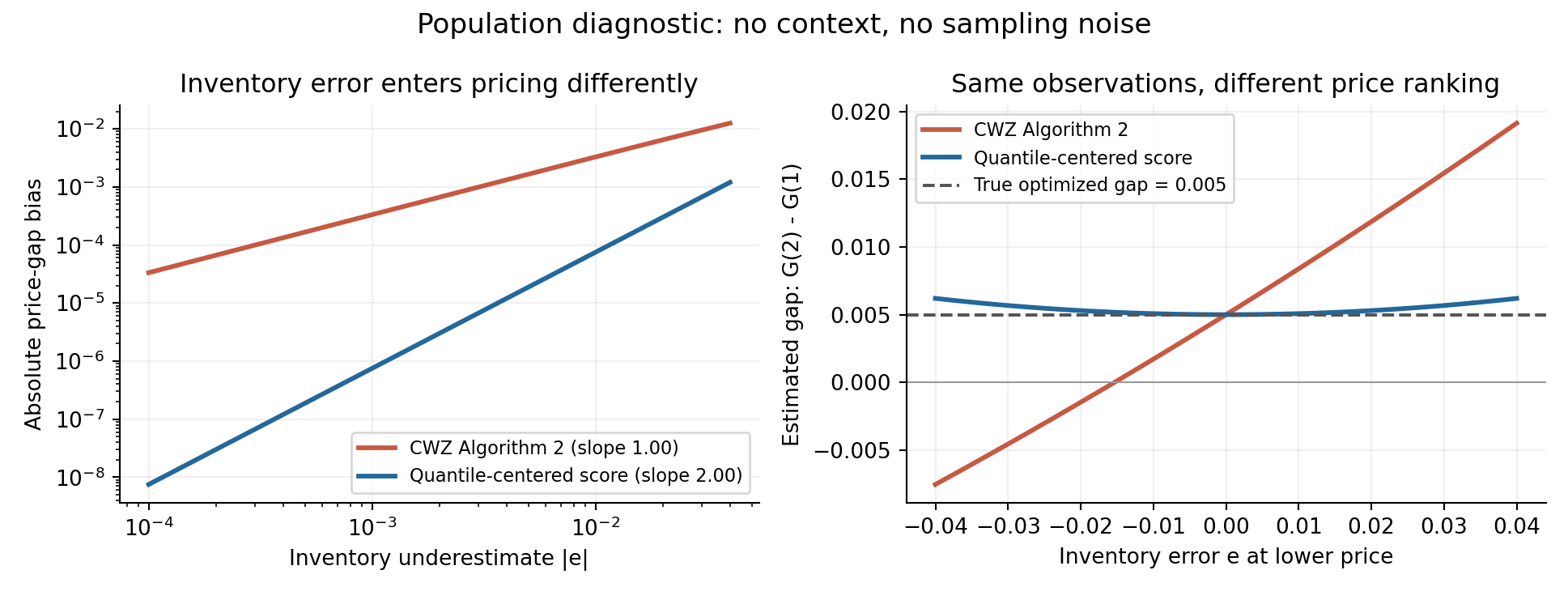}
 \caption{Population comparison with \(e'=0\).  Left: comparison error
 scales linearly with \(|e|\) for CWZ2 and quadratically for SP.
 Right: understocking can reverse CWZ2's price ranking. The true
 optimized-profit gap is \(0.005\).}
 \label{fig:inventory-population-sensitivity}
\end{figure}

\paragraph{Finite samples and a correction control.}
Keep the same demand form, costs, ranges, and candidate prices, and use
\(\xi=2(U-\E U)\) with \(U\) distributed as
\(\operatorname{Beta}(1,1)\), \(\operatorname{Beta}(2,2)\), or
\(\operatorname{Beta}(2,5)\).  For each law, set
\(K_F(p)=(p+2)L(z_p)+z_p\) and
\(c=3+K_F(2)-K_F(1)+0.005\), preserving \(\Delta G=0.005\).
Each price supplies \(n\in\{1024,4096,16384,65536,262144\}\)
independent observations, with 200 replications per law and sample size.
All estimators use the same two independent price streams.  We prescribe
\(e\in\{-0.04,0,+0.04,-0.2n^{-1/4}\}\), always with \(e'=0\),
and select the high price when the estimated gap is positive.

SP uses the empirical proxy mean and low sales quantile.  CWZ2 uses its
original supremum convention, the remaining-inventory order statistic
with one-based index \(\lfloor n\tau(p)\rfloor+1\).
A modified CWZ2 control replaces \(\tau(p)\) by the actual empirical
low-price stockout frequency in both inversion and correction.
For \(a'>a\), its population correction becomes \(b[L(a)-L(a')]\),
removing the additional error in \cref{eq:cwz-extra-correction}.
This control isolates the bias mechanism and differs from published CWZ2.

\begin{figure}[H]
 \centering
 \includegraphics[width=\textwidth]{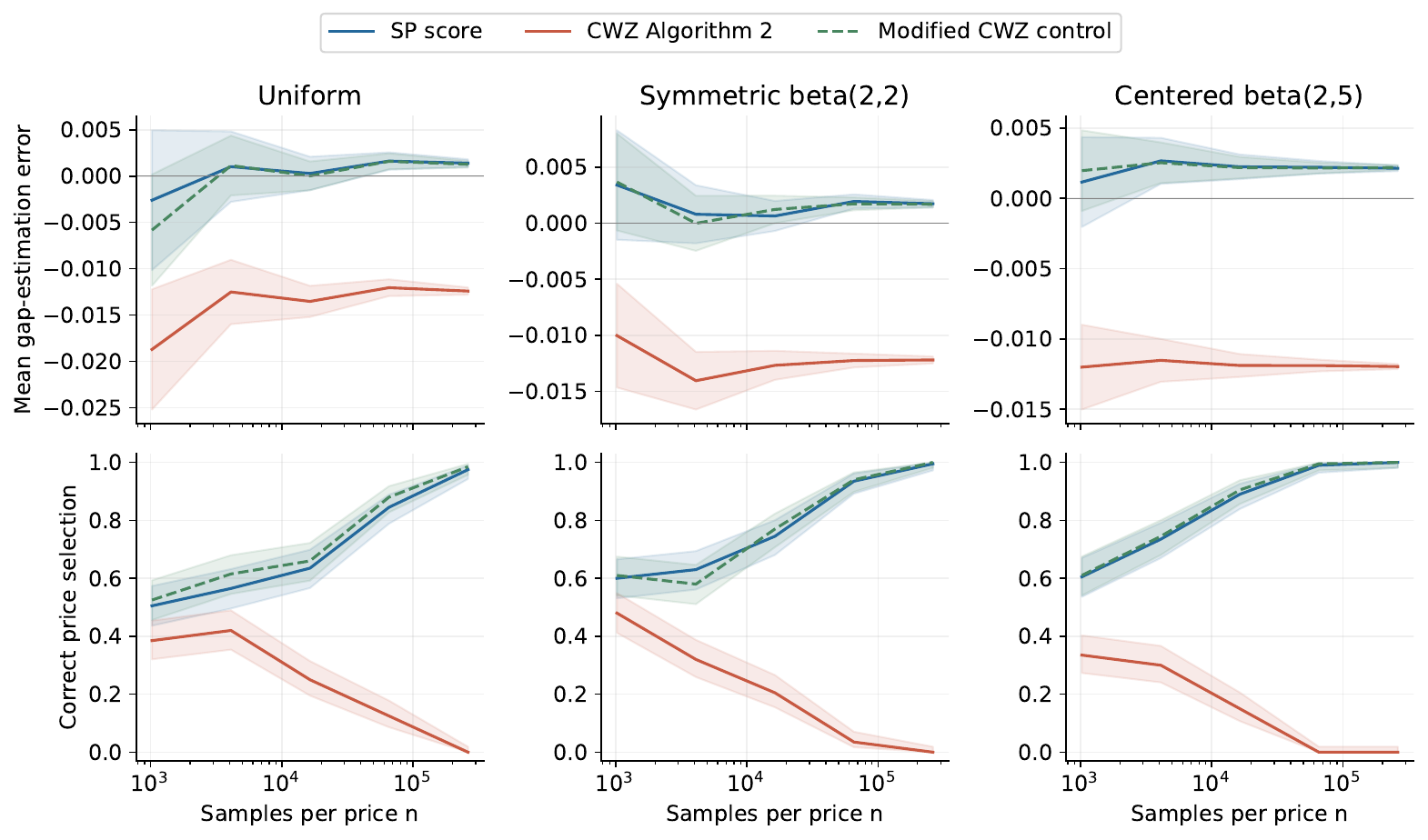}
 \caption{Fixed inventory error \(e=-0.04,e'=0\), using identical data
 across estimators.  Top: mean gap error with 95\% Monte Carlo intervals.
 Bottom: correct-price selection with 95\% Wilson intervals over 200
 replications.  The modified control uses the actual stockout frequency.}
 \label{fig:inventory-finite-sample}
\end{figure}

\begin{table}[H]
 \centering
 \small
 \begin{tabular}{llrrr}
\toprule
Noise & Estimator & Population bias & Mean bias & Correct (\%) \\
\midrule
Uniform & SP score & +0.00120 & +0.00139 & 97.5 \\
Uniform & CWZ Algorithm 2 & -0.01253 & -0.01241 & 0.0 \\
Uniform & Modified CWZ control & +0.00120 & +0.00127 & 98.5 \\
Beta(2,2) & SP score & +0.00172 & +0.00173 & 99.5 \\
Beta(2,2) & CWZ Algorithm 2 & -0.01219 & -0.01219 & 0.0 \\
Beta(2,2) & Modified CWZ control & +0.00172 & +0.00168 & 100.0 \\
Beta(2,5) & SP score & +0.00235 & +0.00214 & 100.0 \\
Beta(2,5) & CWZ Algorithm 2 & -0.01177 & -0.01196 & 0.0 \\
Beta(2,5) & Modified CWZ control & +0.00235 & +0.00220 & 100.0 \\
\bottomrule
\end{tabular}

 \caption{Fixed understocking, \(e=-0.04,e'=0\), at \(n=262144\)
 observations per price.  Population biases are exact. Empirical results
 use 200 replications.}
 \label{tab:inventory-finite-sample}
\end{table}

\Cref{fig:inventory-finite-sample,tab:inventory-finite-sample} show the
finite-sample effect of the population bias under fixed understocking.
As the sample size increases, CWZ2 concentrates around the biased price
difference, and the ranking error persists as sampling noise vanishes.
The modified control traces this bias to using the critical stockout
probability in the correction at a misspecified inventory.  Our visible
low-quantile correction remains valid at that inventory, leaving only the
quadratic inventory profit loss in the price comparison.  This is the
mechanism that allows coarse inventory estimates to support finer price
learning.

\section{AI Use Statement}

Generative AI tools assisted with editing and manuscript organization.  The authors independently
wrote the proofs and take
responsibility for the manuscript.

\end{document}